 \documentclass[final,3p,times]{elsarticle}

\usepackage{amssymb}
\usepackage{amsmath}
\usepackage{amsthm}
\usepackage{subcaption}
\usepackage{color}
\usepackage{bm}

\newtheorem{theorem}{Theorem}
\newtheorem{corollary}{Corollary}
\newtheorem{definition}{Definition}
\newtheorem{lemma}{Lemma}
\newtheorem{proposition}{Proposition}
\newtheorem{remark}{Remark}

\newcommand{\vtheta}{\boldsymbol{\theta}}

\newcommand{\vb}{\boldsymbol{b}}

\newcommand{\vi}{\boldsymbol{i}}

\newcommand{\vl}{\boldsymbol{l}}

\newcommand{\vx}{\boldsymbol{x}}
\newcommand{\vy}{\boldsymbol{y}}
\newcommand{\vz}{\boldsymbol{z}}
\newcommand{\vA}{\boldsymbol{A}}

\newcommand{\vW}{\boldsymbol{W}}

\newcommand{\vZ}{\boldsymbol{Z}}

\newcommand{\sF}{\mathbb{F}}

\newcommand{\sH}{\mathbb{H}}

\newcommand{\sN}{\mathbb{N}}

\newcommand{\sR}{\mathbb{R}}

\newcommand{\sV}{\mathbb{V}}

\newcommand{\sX}{\mathbb{X}}

\newcommand{\fE}{\mathcal{E}}
\newcommand{\fF}{\mathcal{F}}
\newcommand{\fG}{\mathcal{G}}

\newcommand{\fN}{\mathcal{N}}
\newcommand{\fO}{\mathcal{O}}

\newcommand{\fS}{\mathcal{S}}

\newcommand{\fX}{\mathcal{X}}
\newcommand{\fY}{\mathcal{Y}}

\journal{Journal of Complexity}

\begin{document}

\begin{frontmatter}



\title{On the Dimension-Free Approximation of Deep Neural Networks for Symmetric Korobov Functions}

\author[label1]{Yulong Lu}
\ead{yulonglu@umn.edu}

\affiliation[label1]{organization={School of Mathematics, University of Minnesota},
            city={Minneapolis},
            postcode={55455},
            state={MN},
            country={USA}}

\author[label2]{Tong Mao}
\ead{tong.mao@kaust.edu.sa}

\affiliation[label2]{organization={Computer, Electrical and Mathematical Science and Engineering Division,\\ King Abdullah University of Science and Technology, Thuwal 23955, Saudi Arabia}}

\author[label2]{Jinchao Xu}
\ead{jinchao.xu@kaust.edu.sa}

\author[label3]{Yahong Yang}
\ead{yyang3194@gatech.edu}

\affiliation[label3]{organization={School of Mathematics, Georgia Institute of Technology},
            city={Atlanta},
            postcode={30332},
            state={GA},
            country={USA}}
\begin{abstract}
Deep neural networks have been widely used as universal approximators for functions with inherent physical structures, including  permutation symmetry. In this paper,  
we construct symmetric deep neural networks to approximate symmetric Korobov functions  and prove that both the convergence rate and the constant prefactor scale at most polynomially with respect to the ambient dimension.  
This represents a substantial improvement over prior approximation guarantees that suffer from the curse of dimensionality. Building on these approximation bounds, we further derive a generalization-error rate for learning symmetric Korobov functions whose leading factors likewise avoid the curse of dimensionality.
\end{abstract}





\end{frontmatter}



\section{Introduction}
In this paper, we study quantitative approximation of symmetric functions using deep neural networks (DNNs). A symmetric  $f: \boldsymbol{x}= (x_1,\cdots, x_d) \mapsto f(\boldsymbol{x})$ is said {\em symmetric} if $f(x_{\sigma(1)},\cdots, x_{\sigma(d)}) = f(x_1,\cdots, x_d)$ for any permutation $\sigma \in S(d)$. Symmetric functions are fundamental in science and engineering.  For example, in quantum mechanics,  wave functions are permutation-symmetric of the locations of the identical particles in boson systems. In quantum chemistry and  materials science \cite{behler2007generalized,zhang2018deep,dusson2022atomic}, the interatomic potentials are symmetric with respect to the  permutation of the atoms of the
same chemical species. In many applications like aforementioned cases, the symmetric functions involve large number of variables.  Approximation of these high dimensional functions often suffers from the fundamental issue of curse of dimensionality: the number of parameters of the approximator to achieve an $\epsilon$-accuracy grows at the order $O(e^{-c/\epsilon})$. 
Deep neural networks offer various inherent advantages over traditional approximation approaches with great flexibility of encoding symmetry and other invariant structures. A number of techniques have been proposed recently to construct neural network architectures with the symmetry constraint, including Deep Sets \cite{zaheer2017deep,qi2017pointnet,yang2025statistical}, self-attention networks \cite{vaswani2017attention,lee2019set}, architectures with Janossy pooling \cite{murphy2019janossy} and flow-based models \cite{li2024deep}. 

Despite the empirical success of  various DNNs approximating symmetric functions,  theoretical underpinnings of the symmetric neural networks  are far from complete. The recent approximation results \cite{zaheer2017deep,sannai2019universal,yarotsky2022universal,li2024deep,hutter2020representing}  established the universal approximation of symmetric DNNs for symmetric functions, but these results are qualitative in nature. The  paper \cite{han2022universal} obtained a quantitative approximation bound for  symmetric differentiable functions, but their bound is pessimistic and requires $O((\frac{C}{\epsilon})^d)$ to achieve an $\epsilon$ approximation error when $d\gg 1$. More recently, \cite{takeshita2025approximation} proves quantitative bounds for approximating column-symmetric polynomials with transformers; in principle, these results could be combined with quantitative approximations of symmetric polynomials to extend to broader symmetric classes, although a complete end-to-end argument still appears to be missing. In parallel, \cite{bachmayr2024polynomial} derives quantitative bounds for approximating an alternative notion of a symmetric Korobov class via symmetric polynomials, building on the atomic cluster expansion (ACE) framework \cite{drautz2019atomic,dusson2022atomic}. Our work differs from \cite{bachmayr2024polynomial} in several key aspects. First, their symmetric Korobov class is smooth and forms a strictly smaller subset than the (generally rougher) Korobov class considered here.  Second, we study approximation by symmetric neural networks, whereas \cite{bachmayr2024polynomial} focuses on symmetric polynomial bases. Third, at a technical level, instead of employing ACE to construct symmetric bases, we directly symmetrize (piecewise-linear) sparse-grid bases and then approximate these with ReLU or squared-ReLU networks. Crucially, this symmetrization reduces the number of symmetric basis functions required within a given sparse-grid expansion relative to their unsymmetrized counterparts, a saving that is essential for mitigating the curse of dimensionality in our error estimates.

\subsection{Notations}
		Let us summarize all basic notations used in the DNNs as follows:
		
		
		
		
		
		\textbf{1}. Assume $\boldsymbol{n}\in\sN_+^n$, then $f(\boldsymbol{n})=\fO(g(\boldsymbol{n}))$ means that there exists positive $C$ independent of $\boldsymbol{n},f,g$ such that $f(\boldsymbol{n})\le Cg(\boldsymbol{n})$ when all entries of $\boldsymbol{n}$ go to $+\infty$.
		
		\textbf{2}. We call the neural networks with activation function $\sigma$ as $\sigma$ neural networks. By slightly  abusing  notations, we define $\sigma:\sR^d\to\sR^d$ as $$\sigma(\boldsymbol{x})=\left[\begin{array}{c}
			\sigma(x_1) \\
			\vdots \\
			 \sigma(x_d)
		\end{array}\right]$$
        for any $\boldsymbol{x}=\left[x_1, \cdots, x_d\right]^T \in\sR^d$.
		
		\textbf{3}. Define $L,N\in\sN_+$, $N_0=d$ and $N_{L+1}=1$, $N_i\in\sN_+$ for $i=1,2,\ldots,L$, then a $\sigma$ neural network $\phi$ with the width $N$ and depth $L$ can be described as follows:\[\boldsymbol{x}=\tilde{\boldsymbol{h}}_0 \stackrel{\boldsymbol{W}_1, \boldsymbol{b}_1}{\longrightarrow} \boldsymbol{h}_1 \stackrel{\sigma}{\longrightarrow} \tilde{\boldsymbol{h}}_1 \ldots \stackrel{\boldsymbol{W}_L, \boldsymbol{b}_L}{\longrightarrow} \boldsymbol{h}_L \stackrel{\sigma}{\longrightarrow} \tilde{\boldsymbol{h}}_L \stackrel{\boldsymbol{W}_{L+1}, \boldsymbol{b}_{L+1}}{\longrightarrow} \phi(\boldsymbol{x})=\boldsymbol{h}_{L+1},\] where $\boldsymbol{W}_i\in\sR^{N_i\times N_{i-1}}$ and $\boldsymbol{b}_i\in\sR^{N_i}$ are the weight matrix and the bias vector in the $i$-th linear transform in $\phi$, respectively, i.e., $\boldsymbol{h}_i:=\boldsymbol{W}_i \tilde{\boldsymbol{h}}_{i-1}+\boldsymbol{b}_i, ~\text { for } i=1, \ldots, L+1$ and $\tilde{\boldsymbol{h}}_i=\sigma\left(\boldsymbol{h}_i\right),\text{ for }i=1, \ldots, L.$ In this paper, an DNN with the width $N$ and depth $L$, means
		(a) The maximum width of this DNN for all hidden layers less than or equal to $N$.
		(b) The number of hidden layers of this DNN less than or equal to $L$.

\subsection{Structure of the paper}
The paper is organized as follows. In Section~\ref{main results}, we present the main results. Theorem~\ref{thm:main-sym-korobov neural network} establishes an $H^1$ approximation bound for symmetric neural networks on symmetric Korobov spaces. Since both the target function and the symmetric neural network satisfy homogeneous Dirichlet boundary conditions, this bound immediately extends from the $H^1$ semi-norm to the full $H^1$ norm. The estimate completely overcomes the curse of dimensionality, including in the coefficient and logarithmic terms. Theorem~\ref{general thm} then provides generalization error bounds corresponding to this approximation result; these bounds are free from the curse of dimensionality and are nearly optimal with respect to the number of sample points.

Section~\ref{sec:nonsymm-korobov} studies approximation of functions in Korobov spaces without symmetry constraints and shows that, in this general setting, the rates retain terms with exponential dependence on the dimension. This motivates the use of symmetric structures. In Sections~\ref{sec:symm-korobov} and \ref{sec: thm1}, we consider symmetric Korobov spaces together with symmetric neural network architectures and show that this symmetry removes the detrimental dimensional dependence, and complete the proof of Theorem~\ref{thm:main-sym-korobov neural network}. In Section~\ref{sec:thm2}, we establish the generalization bounds and prove Theorem~\ref{general thm}.

\section{Set-up and main results}\label{main results}

We first recall the definition of Korobov spaces as follows. 

\begin{definition}[Korobov Space \citep{bungartz2004sparse,korobov2000coulomb}]
   For $2\le p\le +\infty$, denote $\Omega=[0,1]^d$, the Korobov spaces $X^{2,p}(\Omega)$ is defined as  \[X^{2,p}(\Omega)=\left\{f\in L^p(\Omega)\mid f|_{\partial\Omega}=0,D^{\boldsymbol{k}}f\in L^p(\Omega),|\boldsymbol{k}|_\infty\le 2\right\}\] with $|\boldsymbol{k}|_\infty=\max_{1\le j\le d}k_j$ and the norm  
   \[\|f\|_{X^{2, p}(\Omega)} :=\left(\sum_{0 \leq|\boldsymbol{k}|_\infty \leq 2}\left\|D^{\boldsymbol{k}} f\right\|_{L^p(\Omega)}^p\right)^{1 / p}.\]
   We also define the semi-norm
\begin{align}|f|_{2,p}&:=\left\|\frac{\partial^{2d} f}{\partial x_1^2\cdots\partial x_d^2}\right\|_{L^p(\Omega)},\quad \|f\|_E:=\left(\int_{\Omega} \sum_{j=1}^d\left(\frac{\partial f(\vx)}{\partial x_j}\right)^2 \mathrm{~d} \vx\right)^{1 / 2}.
\end{align}
\end{definition}
Note the fundamental difference between the Korobov space $X^{2,p}$ and the classical Sobolev space $W^{2,p}$: functions in the Korobov space $X^{2,p}$ admit $L^p$-weak mixed derivatives up to order $2d$, whereas Sobolev spaces $W^{2,p}$ only consider $L^p$-weak derivatives up to second order. On the other hand, functions in $X^{2,p}$ admit significantly lower regularity compared to those in $W^{2d,p}$ for dimensions $d > 1$. This difference arises because functions in $X^{2,p}$ are only required to have derivatives up to second order in each individual coordinate direction. Alternative definitions of Korobov spaces based on Fourier features can also be found in \cite{shen2010sparse, griebel2007sparse}. Although formulated differently, these definitions share the same fundamental idea: they construct the Korobov spaces such that functions in these spaces admit derivatives up to high order in each coordinate direction. These alternative definitions are equivalent to the one presented above. 

Recall that a function $f$ is said to be symmetric if
\begin{equation}
    f\left(x_{\tau(1)}, \ldots, x_{\tau(d)}\right) = f\left(x_1, \ldots, x_d\right) \quad \forall \boldsymbol{x} \in \Omega, \quad \tau \in S_d,\notag
\end{equation}
where $S_d$ denotes the  group of permutations on $d$ elements. For later reference, we define the symmetric Korobov spaces $X_{\text{sym}}^{2,p}$ as
\begin{equation}\notag
X_{\text{sym}}^{2,p}\left(\Omega\right) := \left\{ f \in X^{2,p}\left(\Omega\right) : f \text{ is symmetric} \right\}.
\end{equation}

\paragraph{\textbf{Approximation}} The first main result of the paper is the following approximation theorem. \begin{theorem}[Neural network approximation of symmetric Korobov functions]
\label{thm:main-sym-korobov neural network}
Let \(f\in X_{\mathrm{sym}}^{2,2}(\Omega)\) and \(n\in\mathbb N_+\). Then there
exist a permutation-invariant squared ReLU neural network
\(\Phi_{\mathrm{ReLU}_2}:\Omega\to\mathbb R\) and a permutation-invariant ReLU
neural network \(\Phi_{\mathrm{ReLU}}:\Omega\to\mathbb R\) such that, for
\(\kappa=1,2\),
\begin{equation}
\|f-\Phi_{\mathrm{ReLU}_\kappa}\|_E
\le
C d
\left(\frac{5}{12}\right)^d
2^{-n}|f|_{2,2}.
\label{eq:main-nn-error}
\end{equation}
Moreover, both networks vanish on \(\partial\Omega\).

\medskip
\noindent
\emph{(i) Squared ReLU network.}
The squared ReLU network \(\Phi_{\mathrm{ReLU}_2}\) can be chosen with
\begin{equation}
N_2 \le C_s d^3 2^{n+d-1},\qquad
L_2 \le \lfloor\log_2 d\rfloor+2,\qquad
\mathcal M_2 \le C_s d^3 2^{n+d-1}\log_2 d,
\label{eq:relu2-complexity}
\end{equation}
where \(N_2\), \(L_2\), and \(\mathcal M_2\) denote its width, depth, and total
number of nonzero parameters, respectively. Here \(C_s>0\) is an absolute
constant independent of \(d\) and \(n\). Equivalently, in terms of
\(\mathcal M_2\),
\begin{equation}
\|f-\Phi_{\mathrm{ReLU}_2}\|_E
\le
C d^4
\left(\frac{5}{6}\right)^d
(\log_2 d)\,
\mathcal M_2^{-1}|f|_{2,2}.
\label{eq:main-square-relu-error-M}
\end{equation}

\medskip
\noindent
\emph{(ii) ReLU network.}
The ReLU network \(\Phi_{\mathrm{ReLU}}\) can be chosen with
\begin{equation}
N_1
\le
C_s d^2 2^{n+d},\qquad
L_1
\le
C_s d^2(1+n),\qquad
\mathcal M_1
\le
C_s d^2 2^{n+d}
\left(
1+\frac{n}{d\log(2d+1)}
\right).
\label{eq:relu-complexity}
\end{equation}
Equivalently, in terms of \(\mathcal M_1\),
\begin{equation}
\|f-\Phi_{\mathrm{ReLU}}\|_E
\le
C d^3
\left(\frac56\right)^d
\frac{\log(2\mathcal M_1)}{\mathcal M_1}
|f|_{2,2}.
\label{eq:main-relu-error-M-simplified}
\end{equation}

\medskip
\noindent
The dimension-dependent prefactors in
\eqref{eq:main-square-relu-error-M} and
\eqref{eq:main-relu-error-M-simplified} do not grow exponentially in \(d\).
\end{theorem}

The large parameters of above neural networks are localized. In the squared ReLU construction,
only the innermost one-dimensional feature approximation blocks depend on the
basis-level accuracy. In the ReLU construction, all parameters are bounded by
\(
C d^d2^{n+d-1},
\)
and the factor \(d^d\) appears only in the final output layers of the product
approximation blocks. The detail can be found in Propositions~\ref{prop:square-relu-sym-basis} and \ref{prop:sym-basis-relu}.

\begin{remark}
Since both the symmetric neural network approximation and the target function satisfy 
the zero boundary condition, the $H^1$ semi-norm is equivalent to the full 
$H^1$ norm on $H^1_0(\Omega)$. Specifically, for $\Omega = [0,1]^d$, the 
Poincar\'{e} inequality
\[
\|u\|_{L^2(\Omega)} \leq \frac{1}{\pi\sqrt{d}}\|\nabla u\|_{L^2(\Omega)}
\]
holds for all $u \in H^1_0([0,1]^d)$, where the constant $\frac{1}{\pi\sqrt{d}}$ 
is sharp and follows from separation of variables. Consequently,
\[
\|u\|_{H^1(\Omega)} \leq \sqrt{1 + \frac{1}{\pi^2 d}}\, |u|_{H^1(\Omega)},
\]
and the estimate in Theorem~\ref{thm:main-sym-korobov neural network} 
therefore holds with respect to the full $H^1$ norm as well.
\end{remark}

Theorem~\ref{thm:main-sym-korobov neural network} shows that a symmetric
squared ReLU neural network can approximate a symmetric Korobov function on
$\Omega=[0,1]^d$ with rate $\mathcal{O}(m^{-1})$. This rate is
dimension-free in the sense that it contains no factor of the form
$(\log m)^d$, and the leading constant does not grow exponentially with
$d$, which is better than \cite{yang2024near,mao2022approximation,montanelli2019new,suzuki2018adaptivity,yang2026approximation}. We note, however, that this does not mean that all dimension-dependent
effects are completely removed. In the present construction, the network size
$m$ is required to exceed a threshold of order $2^{d-1}$, which comes from
Glynn's formula \cite{glynn2010permanent} used to reduce the symmetrized basis complexity from $d!$ to
$2^{d-1}$. In addition, the bounds on the network parameters may still depend
on $d$, due to the use of ReLU or squared ReLU activations for implementing
multiplication and local basis functions. Thus, the term ``dimension-free''
refers to the approximation error bound itself, namely the rate in $m$ and
the leading prefactor, rather than to the minimal admissible network size or
the parameter magnitudes.

The proof of Theorem \ref{thm:main-sym-korobov neural network} builds upon sparse grid approximation to Korobov functions, which has been well-studied in the literature and will be recalled in the next section. At a high level, our proof follows by first showing that  symmetric Korobov functions achieve an approximation rate of $\mathcal{O}\big(d\big(\frac{5}{12}\big)^{d-1} m^{-1}\big)$ using $m$ symmetric sparse-grid basis functions. In addition, we show that  each symmetric sparse-grid  basis  can be accurately approximated using a squared ReLU or ReLU neural network with \(\mathcal{O}(2^{d-1})\) parameters. These two facts jointly lead to the desired estimate.


\paragraph{\textbf{Generalization}} Next, we state a generalization result on the proposed neural networks for learning the gradient of a symmetric Korobov function from a finite number of samples. To this end, let $\rho$ be a probability measure on the product space
\(
  \mathcal Z := \Omega \times \mathcal Y
\)
with $\mathcal Y\subset\mathbb R^d$.  
Assume for simplicity that the $\vx$–marginal of $\rho$ is the uniform distribution on $\Omega$ and that the conditional mean associated to the conditional distribution $\rho(\mathrm d \vy \,|\,\vx)$ is given by a gradient field, namely there exists a function $f_{\rho}:\Omega \mapsto \mathbb{R}$ such that
\[
  \nabla f_{\rho}(\vx)
  \;=\;
  \int_{\mathcal Y} \vy \,\rho(\mathrm d \vy \,|\,\vx),
  \qquad
  f_{\rho}\in X^{2,2}_{\mathrm{sym}}(\Omega).
\]
Assume that 
\(
  \mathcal S
  =\bigl\{(\vx_{j},\vy_{j})\bigr\}_{j=1}^{M}\subset\mathcal Z^{M}
\)
are i.i.d. samples distributed according to $\rho$.  
Define the empirical  loss of a hypothesis gradient field $\nabla f$ associated to a symmetric function $f$ by 
\begin{equation}
  \mathcal E_{\mathcal S}(f)
  :=\frac1M\sum_{j=1}^{M}\bigl(\nabla f(\vx_{j})-\vy_{j}\bigr)^{2}.\label{equ:loss}
\end{equation}
Moreover, we define the symmetric neural network class
\begin{align*}
\mathcal F_{m,L,C_*}^1:=
    \Bigl\{
    \phi &= \sum_{i=1}^{m} s_i \,\Bigm|\,
    \|\nabla \phi(\vx)\|_{L^\infty(\Omega)} \le L, \ s_i \text{ is a ReLU network } \text{of width } C_*d^{2}
    \text{ and depth } C_*d^2(\log m+d\log d)
  \Bigr\},\\\mathcal F_{m,L,C_*}^2
:=\Bigl\{
    \phi &= \sum_{i=1}^{m} s_i \,\Bigm|\,
    \|\nabla \phi(\vx)\|_{L^\infty(\Omega)} \le L, \ s_i \text{ is a squared ReLU network } \text{of width } C_*d^{2}
    \text{ and depth } \lfloor \log_{2} d \rfloor + 2
  \Bigr\}
\end{align*}
and
\[
\partial_k \fF_{m,L,C_*}^\kappa
:= \{\partial_k \phi \mid \phi \in \fF_{m,L,C_*}^\kappa\}, \qquad k=1,\dots,d,\quad\kappa=1,2.
\]
The classes $\fF_{m,L,C_*}^\kappa$ denote the neural-network hypothesis space
that attains the approximation rate stated in
Theorem~\ref{thm:main-sym-korobov neural network} with an appropriate choice
of $L$; see Corollary~\ref{cor:gen-error-form} for details.

Let $f_{\mathcal S,\fF_{m,L,C_*}^\kappa}$ be an empirical risk minimizer over $\fF_{m,L,C_*}^\kappa$, i.e.  
$$
  f_{\mathcal S,\fF_{m,L,C_*}^\kappa}
  \in \arg\min_{f\in\fF_{m,L,C_*}^\kappa}\mathcal E_{\mathcal S}(f)
$$
Our second main theorem stated below bounds the generalization error
\[
  \mathbb E\bigl\|f_{\mathcal S,\fF_{m,L,C_*}^\kappa}-f_{\rho}\bigr\|_{E},
\]
where the expectation is taken with respect to the joint law  of the samples $S$. 
\begin{theorem}\label{general thm}
Let \(\kappa\in\{1,2\}\), \(f_\rho\in X^{2,2}_{\operatorname{sym}}(\Omega)\), and assume that
\(\|\nabla f_\rho\|_{L^\infty(\Omega)}\le B\). Suppose also that
\(\mathcal Y\in[-B,B]^d\) almost surely. Assume that the approximation error
over \(\fF_{m,L,C_*}^\kappa\) achieves the rate from
Theorem~\ref{thm:main-sym-korobov neural network}, namely
\(\inf_{f\in\fF_{m,L,C_*}^\kappa}\|f-f_\rho\|_E
\le C m^{-1}|f_\rho|_{2,2}\). Then, for \(M=\lceil m^3(\log m)^4\rceil\),
\[
\mathbb E
\left\|
f_{\mathcal S,\fF_{m,L,C_*}^\kappa}-f_\rho
\right\|_E^2
\le
C
\left(
\frac{(\log M)^4}{M}
\right)^{2/3}.
\]
Here \(C>0\) depends at most polynomially on \(d\), \(B\), \(L\), \(C_*\), and
\(|f_\rho|_{2,2}\).
\end{theorem}
\begin{remark}
Note that the bounded gradient assumption \(
  \|\nabla f_\rho\|_{L^\infty(\Omega)} \le L
\) in Theorem~\ref{general thm} is reasonable thanks to the Sobolev embedding
\(X^{2,2}(\Omega) \hookrightarrow C^{1}(\Omega)\)  
(see \cite{schmeisser1987topics}).  
It is therefore natural to assume that the observed derivative data are uniformly bounded, and to enforce the same bound on the derivatives of the neural network that approximates \(f_\rho\).
\end{remark}

Compared with the bound in \cite{suzuki2018adaptivity}, our estimate improves the logarithmic factor while preserving the same polynomial rate in $M$. More precisely, \cite[Theorem 3]{suzuki2018adaptivity} yields the generalization error
\[
M^{-\frac{2s}{2s+1}}(\log M)^{d+1}
\]
when choosing $u = 1 - 1/q = 1/2$.
In our setting it is natural to take $s=1$, since the target function has Sobolev regularity $2$ and the norm in which the error is measured has order $1$, so the difference of smoothness indices is exactly one. Thus our rate is nearly optimal in the polynomial term, while the logarithmic factor is significantly sharper: our bound has no explicit dependence on $d$ in the logarithmic term.

This improvement comes from two key ingredients. First, the approximation error in Theorem~\ref{thm:main-sym-korobov neural network} does not contain any extra logarithmic factor, due to the use of an energy based sparse grid construction, which is defined in \eqref{eq:X}. Second, the squared ReLU activation allows us to keep the network depth independent of the target accuracy, which avoids additional logarithmic contributions from depth scaling.
Lastly, our leading constant depends only polynomially on the dimension $d$—a consequence of the symmetric structure built into the approximation—which eliminates the usual exponential blow‑up and thus avoids the curse of dimensionality.

\begin{remark}
The learning setup in \eqref{equ:loss} differs from standard neural network regression problems. Here, the training data encode information about the gradient of the target function rather than the function itself. This formulation underpins several important applications. One prominent example  is the score-based diffusion model \cite{song2020score}, where the training proceeds by minimizing the denoising score-matching loss:
\begin{equation}
\min_{\mathbf{s}} \sum_{i=1}^N \sigma_i^2
\mathbb{E}_{p_{\text{data}}(\vx)}
\mathbb{E}_{p_{\sigma_i}(\tilde{\vx}\mid \vx)}
\Big[ \big\| \mathbf{s}(\tilde{\vx},\sigma_i)
- \nabla_{\tilde{\vx}} \log p_{\sigma_i}(\tilde{\vx}\mid \vx)
\big\|_2^2 \Big],
\end{equation}
Note that the target score function $\mathbf{s}$ is exactly given by a gradient field $\nabla f$. In applications such as material design \cite{zeni2025generative} and drug discovery \cite{schneuing2024structure}, the relevant score function is again a gradient field whose underlying potential energy is invariant under permutations of atomic or molecular configurations.
Similar gradient-based losses also arise in learning dynamical systems and interaction laws from trajectory data \cite{lu2019nonparametric,li2021identifiability,feng2022learning}, where one seeks interaction forces derived from permutation-invariant interaction potentials.
Our results provide statistical theoretical guarantees for learning the gradient fields arising from those applied problems with structure-preserving neural networks. 
\end{remark}

\section{Sparse grid approximation}\label{sec:nonsymm-korobov}

In this section, we recall some useful facts on sparse grid approximation to Korobov functions, which serve as important ingredients in the proof of our neural network approximation result. Majority of the results presented here can be found in \cite{bungartz2004sparse}.  While this approach has been explored in several prior works  \cite{barthelmann2000high,montanelli2019new,mao2022approximation,suzuki2018adaptivity,yang2024near}, our analysis requires a symmetrized sparse-grid basis and a quantitative evaluation of the benefits of symmetry. We begin by recalling the definition of the sparse grid.

For any $f\in X^{2,p}(\Omega)$, it takes the following representation:  \[f(\boldsymbol{x})=\sum_{\boldsymbol{l}\in\sN^d_+}\sum_{\boldsymbol{i}\in\boldsymbol{i}_{\boldsymbol{l}}}v_{\boldsymbol{l},\boldsymbol{i}}\phi_{\boldsymbol{l},\boldsymbol{i}}(\boldsymbol{x}),\] where \begin{equation}\label{eq:i}
\boldsymbol{i}_{\boldsymbol{l}}:=\left\{\boldsymbol{i} \in\sN^d: \boldsymbol{1} \leq \boldsymbol{i} \leq 2^{\boldsymbol{l}}-\boldsymbol{1}, i_j \text{ odd for all } 1 \leq j \leq d\right\},
\end{equation} and $v_{\boldsymbol{l},\boldsymbol{i}}=\int_{\Omega}\phi_{\boldsymbol{l},\boldsymbol{i}}(\boldsymbol{x})\cdot D^{\boldsymbol{2}} f(\vx)\,\mathrm{d} \vx$ and $\boldsymbol{2}=(2,2,\ldots,2)$. Each sparse grid basis function \(\phi_{\boldsymbol{l},\boldsymbol{i}}\)is a tensor–product of one-dimensional hat functions and  is centered at the dyadic grid point  
\[
  \boldsymbol{x}_{\boldsymbol{l},\boldsymbol{i}}
  \;=\;
  \bigl(x_{l_{1},i_{1}},\dots,x_{l_{d},i_{d}}\bigr)
  \;=\;
  \boldsymbol{i}\odot 2^{-\boldsymbol{l}},
  \qquad 
  \boldsymbol{x}\odot\boldsymbol{y}
  :=(x_{1}y_{1},\dots,x_{d}y_{d}).
\]
We abbreviate the mesh widths by  
\[
  (h_{l_{1}},\dots,h_{l_{d}})
  \;:=\;
  2^{-\boldsymbol{l}}.
\]
 In a piecewise linear setting, the fundamental choice for a 1D basis function is the standard hat function $\phi(x)$, defined as:
\[
\phi(x):= \begin{cases}1-|x|, & \text { if } x \in[-1,1] \\ 0, & \text { otherwise }\end{cases}
\]
The standard hat function $\phi(x)$ can be utilized to generate any $\phi_{l_j, i_j}\left(x_j\right)$ with support $\left[x_{l_j, i_j}-h_{l_j}, x_{l_j, i_j}+h_{l_j}\right]$ through dilation and translation:
\[
\phi_{l_j, i_j}\left(x_j\right):=\phi\left(\frac{x_j-i_j \cdot h_{l_j}}{h_{l_j}}\right) .
\]
The resulting 1D basis functions serve as inputs for the tensor product construction, yielding a suitable piecewise $d$-linear basis function at each grid point $\boldsymbol{x}_{\boldsymbol{l}, \boldsymbol{i}}$
\[
\phi_{\boldsymbol{l}, \boldsymbol{i}}(\boldsymbol{x}):=\prod_{j=1}^d \phi_{l_j, i_j}\left(x_j\right).
\]



  In the approximation results of \cite{montanelli2019new,mao2022approximation,suzuki2018adaptivity,yang2024near}, the term $(\log_2 m)^{d-1}$ consistently appears in the error bounds, where $m$ denotes the number of network parameters and $d$ the input dimension. Although the dominant polynomial term in $m$ does not exhibit the curse of dimensionality, the logarithmic factor does. Consequently, these results do not fully overcome the curse of dimensionality. This limitation arises because their approximation results builds upon the following lemma:

\begin{lemma}[\cite{bungartz2004sparse}, Lemma 3.13]\label{err-first}
For any $f \in X^{2, 2}(\Omega)$, define
\[
f_n^{(1)}(\boldsymbol{x}) = \sum_{|\boldsymbol{l}|_1 \leq n+d-1} \sum_{\boldsymbol{i} \in \boldsymbol{i}_{\boldsymbol{l}}} v_{\boldsymbol{l}, \boldsymbol{i}} \phi_{\boldsymbol{l}, \boldsymbol{i}}(\boldsymbol{x}),
\]with $v_{\boldsymbol{l},\boldsymbol{i}}=\int_{\Omega}\phi_{\boldsymbol{l},\boldsymbol{i}}(\boldsymbol{x})\cdot D^{\boldsymbol{2}} f(\vx)\,\mathrm{d} \vx$ and $\boldsymbol{2}=(2,2,\ldots,2)$, the following approximation error estimates hold:
\begin{align}
\|f - f_n^{(1)}\|_{E} &= \mathcal{O}\left(M^{-1}(\log_2M)^{d-1}\right),
\end{align}
where $M$ is the number of the basis in $\sV_n:=\{\boldsymbol{l} \in \mathbb{N}^d \mid|\boldsymbol{l}|_1 \leq n+d-1\}$, and $\|\cdot\|_E$ is the energy norm defined in Theorem 1. 
\end{lemma}

Notice that by construction, our symmetric neural network approximator $f_n^{(1)}$ vanishes on the boundary of $\Omega$ as the target function $f$. We chose to use the energy norm  \(\|\cdot\|_E\) as it is equivalent to the $H^1$-norm in $H_0^1(\bar{\Omega})$ owing to the Poincar\'e's inequality.  To eliminate the curse of dimensionality, our first objective is to remove the $(\log_2 M)^{\,d-1}$ factor from error. We achieve this by replacing the conventional total-degree index set \(\sV_n\)
with the more economical {energy-based sparse-grid} index set described in~\cite[Theorem~3.10]{bungartz2004sparse}. This refined choice of multi-indices yields a markedly tighter complexity bound and enables us to obtain an improved approximation rate.

\begin{lemma}[\cite{bungartz2004sparse}, Theorem 3.10]\label{err-second}
Define
\[
f_n^{(2)}(\boldsymbol{x}) = \sum_{|\vl|_1-\frac{1}{5} \cdot \log _2\left(\sum_{j=1}^d 4^l_j\right) \leq(n+d-1)-\frac{1}{5} \cdot \log _2\left(4^n+4 d-4\right)} \sum_{\boldsymbol{i} \in \boldsymbol{i}_{\boldsymbol{l}}} v_{\boldsymbol{l}, \boldsymbol{i}} \phi_{\boldsymbol{l}, \boldsymbol{i}}(\boldsymbol{x}).
\]
Then for any $f \in X^{2, 2}(\Omega)$, the following approximation error estimate holds:
\begin{align}
\|f - f_n^{(2)}\|_{E} \le \frac{2 \cdot d \cdot|f|_{2, 2}}{\sqrt{3} \cdot 6^{d-1}} \cdot\left(\frac{1}{2}+\left(\frac{5}{2}\right)^{d-1}\right)2^{-n}.\notag
\end{align}
\end{lemma}
For notational convenience, we define the energy-based index set $\sX_n$ by setting
\begin{equation}
\sX_n := \left\{\boldsymbol{l} \in \mathbb{N}^d \mid |\boldsymbol{l}|_1 - \tfrac{1}{5} \log_2\left(\sum_{j=1}^d 4^{l_j}\right) \leq (n+d-1) - \tfrac{1}{5} \log_2\left(4^n + 4d - 4\right) \right\}.\label{eq:X}
\end{equation}
\begin{lemma}[\cite{bungartz2004sparse}, Lemma 3.9]\label{XN}
The energy-based sparse grid index set $\sX_n$ is a subset of $\{ \boldsymbol{l} \in \mathbb{N}^d : |\boldsymbol{l}|_1 \leq n+d-1 \}$, and its cardinality satisfies
\[
\left|\bigcup_{\vl\in\sX_n}\vi_{\vl}\right| \leq 2^n \cdot \frac{d}{2} \cdot \mathrm{e}^d.
\]
\end{lemma}
Based on Lemmas~\ref{err-second} and~\ref{XN}, this refined index set further removes the residual logarithmic dependence on the dimension.  
In comparison with the classical sparse–grid set $\sV_n$, our new set $\sX_n$ contains markedly fewer multi-indices.  
The guiding principle is a {cost–benefit} analysis: for every candidate index in $\sV_n$ we evaluate the ratio between its contribution to the approximation accuracy and the extra computational cost it entails; indices with an unfavorable ratio are discarded.  
This pruning strategy keeps the cardinality of the index set small while leaving the overall approximation error essentially unchanged.  Figure~\ref{indexdiff} contrasts $\sX_n$ and $\sV_n$ for the illustrative case $d=2$ and $n=5$.
 \begin{figure}[!th]
\centering
\begin{subfigure}[t]{0.50\textwidth}
    \centering
    \includegraphics[width=0.99\textwidth]{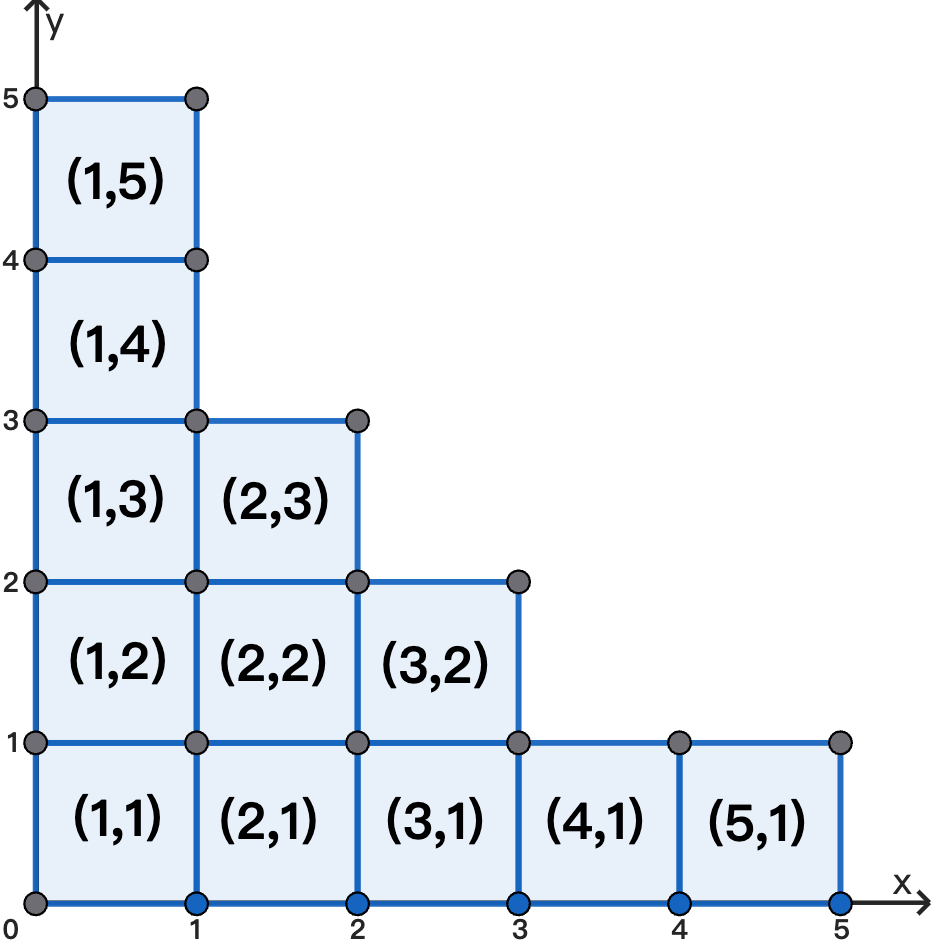}
    \caption{$\sX_5$ for $d=2$}
    \label{fig:compare-a}
\end{subfigure}%
\hfill
\begin{subfigure}[t]{0.50\textwidth}
    \centering
    \includegraphics[width=0.99\textwidth]{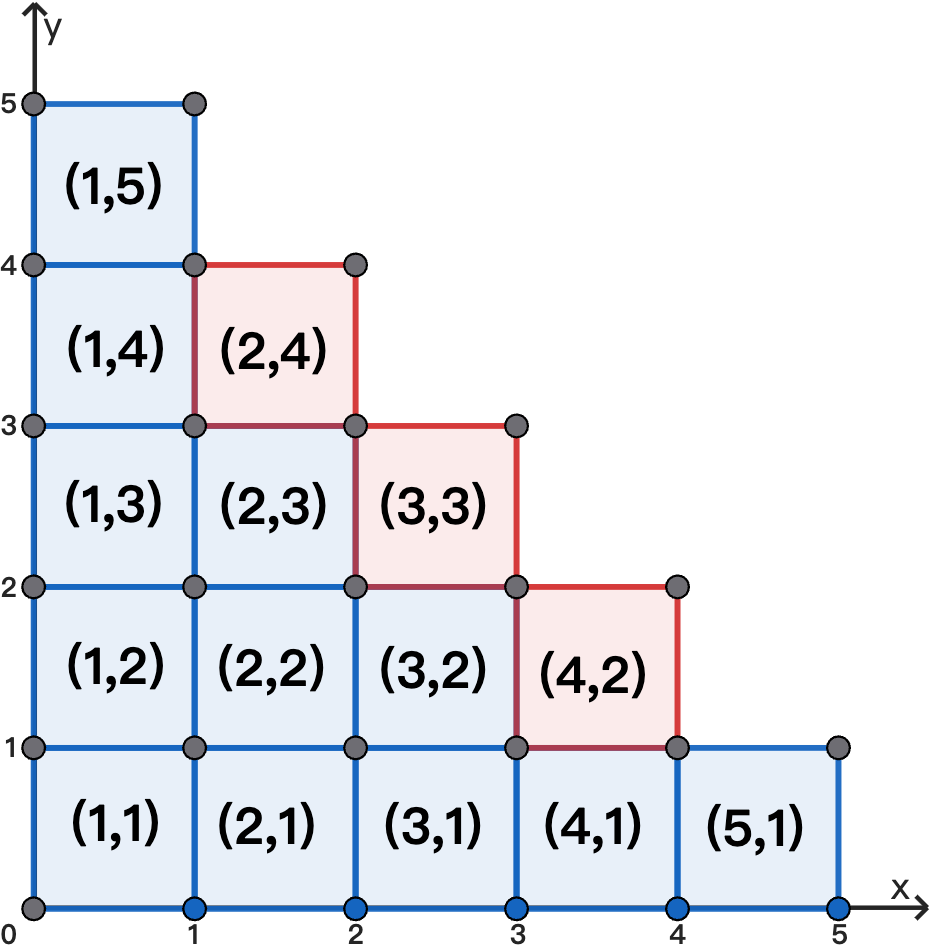}
    \caption{$\sV_5$ for $d=2$}
    \label{fig:compare-b}
\end{subfigure}
\caption{The difference between $\sX_n$ and $\sV_n$ for $n=5$ and $d=2$.}
\label{indexdiff}
\end{figure}The cost–benefit analysis underlying this choice of indices is detailed in~\cite{bungartz2004sparse,griebel2000optimized}.  
A similar construction is available for Korobov spaces with global Fourier bases; see~\cite{griebel2007sparse,shen2010sparse}.  
However, Corollary~\ref{nonbetter} below reveals that—even with this refinement—the cardinality of the index set, and thus the leading constant in the approximation error, still grows exponentially with the dimension~$d$.

Combining Lemmas~\ref{err-second} and~\ref{XN}, we obtain the following result:

\begin{corollary}\label{nonbetter}
For any $f \in X^{2, 2}(\Omega)$, the approximation $f_n^{(2)}(\boldsymbol{x})$ defined in Lemma~\ref{err-second} uses fewer than $m := 2^n \cdot \frac{d}{2} \cdot \mathrm{e}^d$ basis functions and achieves the approximation rate
\begin{align}
\|f - f_n^{(2)}\|_{E} \leq \frac{2e \cdot d^2 \cdot |f|_{2, 2}}{\sqrt{3}} \cdot \left(\frac{5\mathrm{e}}{12}\right)^{d-1} m^{-1}.\notag
\end{align}
\end{corollary}

Observe that while the logarithmic factor in the approximation error is mitigated, the prefactor in front of $m^{-1}$ still still grows expoenentially in the dimension $d$  as $\frac{5\mathrm{e}}{12} > 1$. This is mainly due to that the size of $\sX_n$ grows exponentially in  $d$. This phenomenon is also confirmed by numerical simulations, such as those reported in \cite[Table 3.1]{bungartz2004sparse}, where exponential growth in $d$ is clearly observed. Similar strategies aimed at eliminating the logarithmic term using alternative formulations of Korobov spaces—especially those based on Fourier features—have been explored in \cite{griebel2007sparse,shen2010sparse}. However, such approaches also result in constants that grow exponentially in $d$. To overcome this barrier, we shift our focus to symmetric Korobov spaces.  
As we will demonstrate in Section~\ref{sec:symm-korobov}, the permutation-symmetry enables us to avoid the  the curse of dimensionality in both the prefactor and the convergence rate.

\section{Symmetric sparse grid approximation for symmetric Korobov functions}\label{sec:symm-korobov}


Based on Lemma~\ref{XN}, we have the following result:

\begin{lemma}\label{boundconstantsym}
Define
\[
f_n^{(2)}(\boldsymbol{x}) = \sum_{\boldsymbol{l} \in \sX_n \cap \mathbb{N}^d_{\operatorname{ord}}} \sum_{\boldsymbol{i} \in \boldsymbol{i}_{\boldsymbol{l}}} \bar{v}_{\boldsymbol{l}, \boldsymbol{i}} \sum_{\tau \in S_d} \phi_{\boldsymbol{l}, \boldsymbol{i}}(\tau(\boldsymbol{x})),
\]
where $\tau(\boldsymbol{x}) = (x_{\tau(1)}, \ldots, x_{\tau(d)})$, and $\bar{v}_{\boldsymbol{l}, \boldsymbol{i}}$ is a suitably rescaled version of $v_{\boldsymbol{l}, \boldsymbol{i}}$, accounting for repeated indices. Here,
\[
\mathbb{N}^d_{\operatorname{ord}} := \left\{ \boldsymbol{l} \in \mathbb{N}^d \mid 1\le l_1 \leq l_2 \leq \cdots \leq l_d \right\}.
\]
Then, for any $f \in X^{2,2}_{\text{sym}}(\Omega)$, the approximation error satisfies
\[
\|f - f_n^{(2)}\|_{E} \le \frac{2 \cdot d \cdot|f|_{2, 2}}{\sqrt{3} \cdot 6^{d-1}} \cdot\left(\frac{1}{2}+\left(\frac{5}{2}\right)^{d-1}\right)2^{-n}.
\]
\end{lemma}

For every tensor–product basis
\(
\phi_{\boldsymbol{l},\boldsymbol{i}},
\)
now we form its symmetrization
\begin{equation}
    \psi_{\boldsymbol{l},\boldsymbol{i}}(\boldsymbol{x})
    :=\sum_{\tau\in S_{d}}
       \phi_{\boldsymbol{l},\boldsymbol{i}}\bigl(\tau(\boldsymbol{x})\bigr),\label{basissys}
\end{equation}
and treat the resulting \(\psi_{\boldsymbol{l},\boldsymbol{i}}\) as a single basis function.
Our first goal is to determine the number of these symmetric bases appear in the truncated expansion
\(f_{n}^{(2)}\).
Concretely, let $\boldsymbol{i}_{\boldsymbol{l}}$ denote the set of grid points at level $\boldsymbol{l}$ as defined in~\eqref{eq:i}.  Our next task is to bound the cardinality of
\[
   \bigcup_{\boldsymbol{l}\in\sX_{n}\cap\mathbb{N}^{d}_{\operatorname{ord}}}
   \boldsymbol{i}_{\boldsymbol{l}}, 
   \qquad
   \mathbb{N}^{d}_{\operatorname{ord}}
   :=\bigl\{\boldsymbol{l}\in\mathbb{N}^{d}\;:\;
           1\le l_{1}\le l_{2}\le\cdots\le l_{d}\bigr\},
\]
which represents the total number of grid points used in the symmetric sparse grid.

Our counting argument follows the framework of~\cite{bungartz2004sparse},  
but it explicitly incorporates the ordering constraint \(l_{1}\le\cdots\le l_{d}\) that reflects permutation invariance.
Once this bound is established, each \(\psi_{\boldsymbol{l},\boldsymbol{i}}\) is approximated by a dedicated neural network,  
and these components are assembled into a neural network to approximate a function belonging to the symmetric Korobov space.

\begin{proposition}
\label{XNsym}
\[
\left|
\bigcup_{\boldsymbol{l}\in\sX_n\cap\mathbb N^d_{\operatorname{ord}}}
\boldsymbol{i}_{\boldsymbol{l}}
\right|
\le
C_s\,2^n,
\]
where \(C_s>0\) is a constant independent of both \(n\) and \(d\).
\end{proposition}

\begin{proof}
The symmetry of \(\sX_n\) follows immediately from its definition, since the defining inequality is invariant under any permutation of the coordinate indices.

By the definition of \(\sX_n\), and using
\[
|\boldsymbol{i}_{\boldsymbol{l}}|
=
2^{|\boldsymbol{l}|_1-d},
\]
we have
\begin{align}
\left|
\bigcup_{\boldsymbol{l}\in\sX_n\cap\mathbb N^d_{\operatorname{ord}}}
\boldsymbol{i}_{\boldsymbol{l}}
\right|
&\le
\sum_{i=0}^{n-1}
\sum_{\substack{
|\boldsymbol{l}|_1=n+d-1-i,\\
\sum_{j=1}^d4^{l_j}\ge (4^n+4d-4)/32^i,\\
1\le l_1\le\cdots\le l_d
}}
|\boldsymbol{i}_{\boldsymbol{l}}|
=
2^{n-1}
\sum_{i=0}^{n-1}
2^{-i}
\sum_{\substack{
|\boldsymbol{l}|_1=n+d-1-i,\\
\sum_{j=1}^d4^{l_j}\ge (4^n+4d-4)/32^i,\\
1\le l_1\le\cdots\le l_d
}}
1.
\label{eq:sym-grid-count-start}
\end{align}

The condition
\[
\sum_{j=1}^d4^{l_j}\ge \frac{4^n+4d-4}{32^i}
\]
implies that, for the ordered representative \(1\le l_1\le\cdots\le l_d\),
\[
l_d=|\boldsymbol{l}|_\infty
\ge
n-\lfloor 2.5 i\rfloor.
\]
Therefore,
\[
l_1+\cdots+l_{d-1}
=
|\boldsymbol{l}|_1-l_d
\le
n+d-1-i-\bigl(n-\lfloor2.5i\rfloor\bigr)
\le
d-1+\lfloor1.5i\rfloor.
\]
Now set
\[
k_j:=l_j-1,\qquad j=1,\dots,d-1.
\]
Then \(k_j\ge0\), \(k_1\le\cdots\le k_{d-1}\), and
\[
k_1+\cdots+k_{d-1}
\le
\lfloor1.5i\rfloor.
\]
Thus, for each fixed \(i\), the number of possible ordered tails
\((l_1,\dots,l_{d-1})\) is bounded by
\[
\sum_{s=0}^{\lfloor1.5i\rfloor} p(s),
\]
where \(p(s)\) denotes the unrestricted partition number of \(s\). This bound is independent of \(d\).

Substituting this estimate into \eqref{eq:sym-grid-count-start}, we obtain
\begin{align}
\left|
\bigcup_{\boldsymbol{l}\in\sX_n\cap\mathbb N^d_{\operatorname{ord}}}
\boldsymbol{i}_{\boldsymbol{l}}
\right|
&\le
2^{n-1}
\sum_{i=0}^{n-1}
2^{-i}
\sum_{s=0}^{\lfloor1.5i\rfloor} p(s)
\le
2^{n-1}
\sum_{i=0}^{\infty}
2^{-i}
\sum_{s=0}^{\lfloor1.5i\rfloor} p(s).
\label{eq:partition-sum-bound}
\end{align}

It remains to note that the last series converges. Indeed, by the Hardy--Ramanujan estimate \cite{hardy1918asymptotic},
\[
p(s)\le C\exp\!\left(\pi\sqrt{\frac{2s}{3}}\right),
\qquad s\ge1.
\]
Hence
\[
\sum_{s=0}^{\lfloor1.5i\rfloor}p(s)
\le
C(1+i)\exp\!\left(\pi\sqrt{i}\right),
\]
after enlarging \(C\) if necessary. Therefore
\[
\sum_{i=0}^{\infty}
2^{-i}
\sum_{s=0}^{\lfloor1.5i\rfloor}p(s)
\le
C
\sum_{i=0}^{\infty}
2^{-i}(1+i)e^{\pi\sqrt{i}}
<\infty.
\]
Consequently, there exists an absolute constant \(C_s>0\), independent of both \(n\) and \(d\), such that
\[
\left|
\bigcup_{\boldsymbol{l}\in\sX_n\cap\mathbb N^d_{\operatorname{ord}}}
\boldsymbol{i}_{\boldsymbol{l}}
\right|
\le
C_s2^n.
\]
This completes the proof.
\end{proof}

Proposition~\ref{XNsym} shows that the number of symmetric basis functions is significantly smaller than that of the standard sparse grid basis without symmetry. By counting the symmetric basis as a single function defined as
\(
    \psi_{\boldsymbol{l}, \boldsymbol{i}} = \sum_{\tau \in S_d} \phi_{\boldsymbol{l}, \boldsymbol{i}}(\tau(\boldsymbol{x})),\)
and leveraging Lemma~\ref{boundconstantsym} and Proposition~\ref{XNsym}, we obtain the following corollary:

\begin{corollary}\label{better}
For any $f \in X^{2, 2}_{\text{sym}}(\Omega)$, the approximation $f_n^{(2)}(\boldsymbol{x})$ defined in Lemma~\ref{boundconstantsym} uses fewer than $m := C_s 2^n $ basis functions and achieves the approximation rate
\begin{align}
\|f - f_n^{(2)}\|_{E} \leq \frac{48 C_s d \cdot |f|_{2, 2}}{5\sqrt{3}} \cdot \left(\frac{5}{12}\right)^d  m^{-1}.
\end{align}
\end{corollary}

\begin{proof}
Set $m = C_s 2^n $, then we have $2^{-n} = m^{-1} C_s $. Using the estimate from Lemma~\ref{boundconstantsym}, we obtain
\begin{align*}
\|f - f_n^{(2)}\|_{E} &\le \frac{2 d \cdot |f|_{2, 2}}{\sqrt{3} \cdot 6^{d-1}} \cdot\left(\frac{1}{2} + \left(\frac{5}{2}\right)^{d-1}\right) 2^{-n} \le \frac{48 d \cdot |f|_{2, 2}}{5\sqrt{3}} \cdot\left(\frac{5}{12}\right)^d C_s  m^{-1}.
\end{align*}
\end{proof}

Comparing Corollary~\ref{better} with Corollary~\ref{nonbetter}, one notices that incorporating the symmetry constraint   significantly improves the error bound. It not only eliminates the curse of dimensionality in the rate but also provides a mechanism for mitigating large values of $|f|_{2,2}$ in high-dimensional settings through the factor $\left(\frac{5}{12}\right)^{d}$.

\section{Proof of Theorem ~\ref{thm:main-sym-korobov neural network}}\label{sec: thm1}
\subsection{Approximating Symmetric Basis Functions}

To prove Theorem~\ref{thm:main-sym-korobov neural network}, our next step is to construct a neural
network approximation of the symmetrized basis function
\[
\psi_{\boldsymbol l,\boldsymbol i}(\vx)
:=
\sum_{\tau\in S_d}
\phi_{\boldsymbol l,\boldsymbol i}\bigl(\tau(\vx)\bigr).
\]
We show below that each \(\psi_{\boldsymbol l,\boldsymbol i}\) can be realized,
up to a prescribed accuracy, by a neural network with
\[
\mathcal O\!\left(\operatorname{poly}(d)\,2^{d-1}\right)
\]
neurons, rather than the factorial complexity \(\mathcal O(d!)\) arising from
the direct symmetrization.

The main ingredient is an alternative representation of the symmetrized
sparse-grid basis. Instead of summing over all \(d!\) permutations, this
representation involves only \(\mathcal O(2^d)\) terms, using Glynn's formula
for the permanent \cite{glynn2010permanent}. This step is essential for making
the later neural-network construction effective.

This construction is different from existing uses of symmetry. The regularity
result of \cite{yserentant2004regularity} shows that eigenfunctions of the
electronic Schr\"odinger equation belong to mixed-derivative Sobolev, or
Korobov, spaces, while \cite{Weimar2012,griebel2007sparse} exploit
permutation symmetry by grouping functions into equivalence classes. These
works, however, do not provide an explicit representation of the symmetrized
local basis suitable for neural-network approximation. Related constructions
for symmetric polynomial or spectral bases
\cite{bachmayr2024polynomial,ho2024atomic,takeshita2025approximation} rely on
closure under multiplication, which is not available for sparse-grid hat
functions. Indeed, the product of two local hat functions is generally not a
linear combination of hat functions of the same type.

Our approach avoids this difficulty by collecting the directional features
\(\{\varphi_j(x_i):1\le i,j\le d\}\) and combining them through the permanent
representation. Glynn's formula expresses each symmetrized product as a signed
linear combination of only \(2^{d-1}\) products of one-dimensional feature
sums, which can then be implemented by neural-network multiplication modules.
This yields an explicit symmetric local sparse-grid construction and leads to
quantitative neural-network approximation rates for symmetric Korobov
functions.

We first recall Glynn's formula.
\begin{lemma}[Theorem 2.1 in \cite{glynn2010permanent}]
\label{lem:glynn-permanent}
Let \(\vA=(a_{i,j})\) be an \(m\times m\) matrix over a field \(F\) of
characteristic not equal to two. Then
\begin{equation}
\operatorname{per}(\vA)
:=
\sum_{\tau\in S_m}\prod_{j=1}^m a_{j,\tau(j)}
=
2^{1-m}
\sum_{\substack{\boldsymbol{\delta}\in\{\pm1\}^m\\ \delta_1=1}}
\left(\prod_{k=1}^m \delta_k\right)
\prod_{j=1}^m
\left(\sum_{i=1}^m \delta_i a_{ij}\right).
\label{eq:glynn-permanent}
\end{equation}
\end{lemma}

Applying Lemma~\ref{lem:glynn-permanent}, we obtain the following representation
of the symmetrized product.

\begin{lemma}[Glynn representation of symmetrized products]
\label{lem:sym-product-glynn}
Let \(d\in\mathbb N\), and let
\(\varphi_1,\dots,\varphi_d:\mathbb R\to\mathbb R\) be arbitrary functions.
Then, for every \(\vx=(x_1,\dots,x_d)\in\mathbb R^d\),
\begin{equation}
\sum_{\tau\in S_d}\prod_{j=1}^d \varphi_j(x_{\tau(j)})
=
2^{1-d}
\sum_{\substack{\boldsymbol{\delta}\in\{\pm1\}^d\\ \delta_1=1}}
\left(\prod_{k=1}^d \delta_k\right)
\prod_{j=1}^d
\left(\sum_{i=1}^d \delta_i\,\varphi_j(x_i)\right).
\label{eq:glynn-sym-product}
\end{equation}
\end{lemma}

\begin{proof}
Define the matrix \(\vA(\vx)=(a_{ij})_{1\le i,j\le d}\) by
\[
a_{ij}:=\varphi_j(x_i).
\]
Then
\[
\operatorname{per}(\vA(\vx))
=
\sum_{\tau\in S_d}
\prod_{j=1}^d a_{\tau(j),j}
=
\sum_{\tau\in S_d}
\prod_{j=1}^d \varphi_j(x_{\tau(j)}).
\]
Applying Lemma~\ref{lem:glynn-permanent} to \(\vA(\vx)\) gives
\[
\operatorname{per}(\vA(\vx))
=
2^{1-d}
\sum_{\substack{\boldsymbol{\delta}\in\{\pm1\}^d\\ \delta_1=1}}
\left(\prod_{k=1}^d \delta_k\right)
\prod_{j=1}^d
\left(\sum_{i=1}^d \delta_i a_{ij}\right).
\]
Since \(a_{ij}=\varphi_j(x_i)\), this is exactly
\eqref{eq:glynn-sym-product}.
\end{proof}

In this section, we study the approximation, in the \(H^1\)-seminorm, of the
Glynn-type representation by neural networks. The main technical component is
the realization or approximation of products. Different activation functions
lead to different product constructions. We first consider squared ReLU neural
networks, for which multiplication can be represented exactly.

\begin{lemma}[Exact product realization by squared ReLU networks]
\label{lem:product-square-relu}
Let \(\sigma_2(x):=\max\{x,0\}^2\). For any integer \(d\ge2\), the product map
\((x_1,\dots,x_d)\mapsto \prod_{j=1}^d x_j\) can be realized exactly by a
squared ReLU network \(\Phi_d\). Moreover, one can choose \(\Phi_d\) with depth
\[
L\le \lfloor\log_2 d\rfloor+1,
\]
width
\[
N\le 2d,
\]
and total numbers of neurons and nonzero parameters bounded by \(\mathcal O(d)\).
All weights and biases can be chosen with magnitude bounded by an absolute
constant.
\end{lemma}

\begin{figure}
    \centering
    \includegraphics[width=0.97\linewidth]{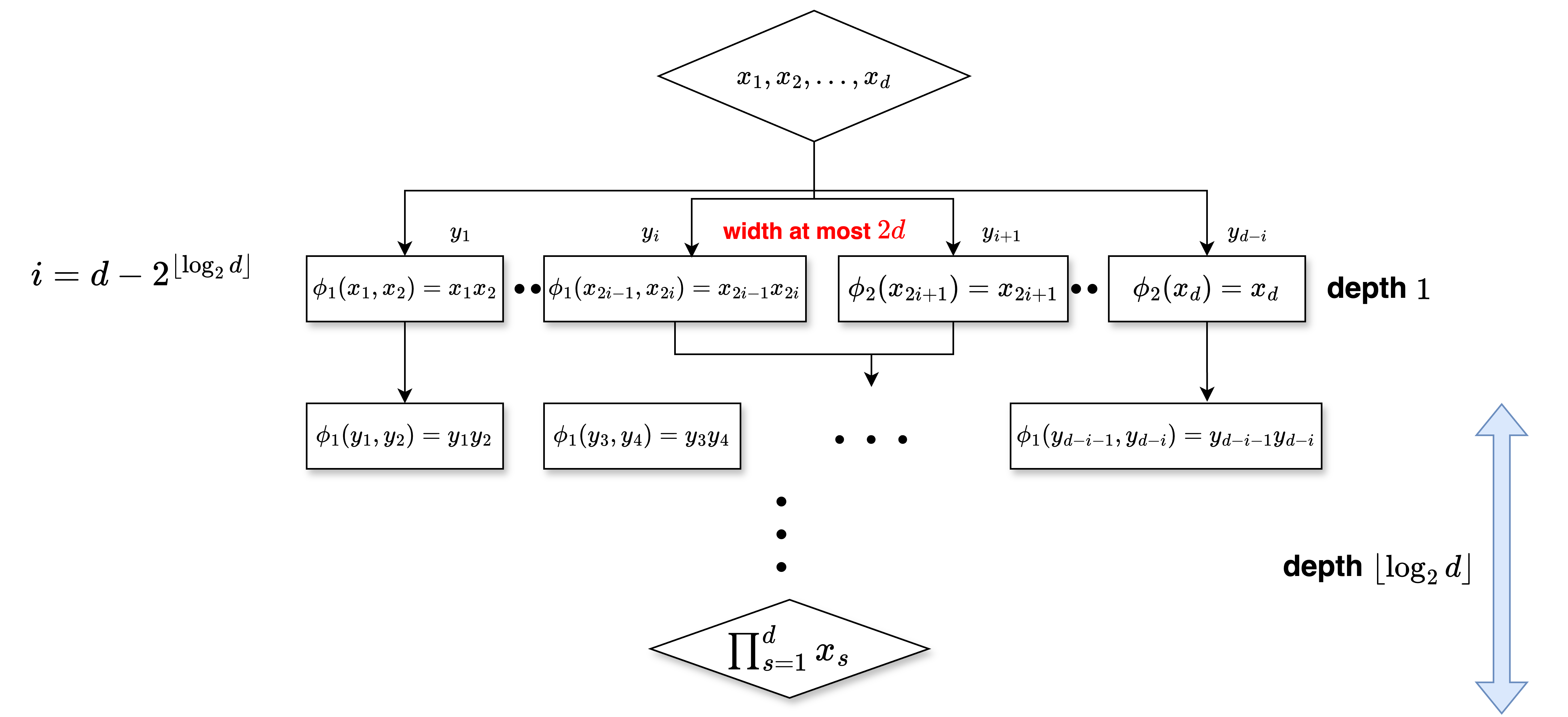}
    \caption{Binary-tree realization of the product \(\prod_{j=1}^d x_j\).}
    \label{fig:prod-square-relu}
\end{figure}

\begin{proof}
The squared ReLU activation satisfies \(x^2=\sigma_2(x)+\sigma_2(-x)\).
Therefore, using the identity \(xy=((x+y)^2-x^2-y^2)/2\), the bilinear map
\((x,y)\mapsto xy\) can be realized exactly by a depth-one, width-three squared
ReLU block. The identity channels can be passed forward by affine connections;
if one uses a strictly layered realization without skip connections, this can be
implemented by a fixed-size squared ReLU identity block. This only changes the
constants in the width and parameter bounds.

Let \(k:=2^{\lfloor\log_2 d\rfloor}\) and \(i:=d-k\). Then \(0\le i<k\), and
\(k\) is the largest power of two not exceeding \(d\). The first layer
compresses the list \((x_1,\dots,x_d)\) of length \(d\) to a list of length
\(k\) as follows. We perform \(i\) adjacent multiplications
\((x_1,x_2),(x_3,x_4),\dots,(x_{2i-1},x_{2i})\), each realized by the bilinear
block above, and pass the remaining inputs unchanged through identity blocks.
Each multiplication reduces the list length by one, so after these \(i\)
multiplications the number of outputs is \(d-i=k\). This compression layer uses
at most \(3i+2(d-2i)\le 2d\) neurons and \(\mathcal O(d)\) parameters.

After this step, the number of active factors is exactly
\(k=2^{\lfloor\log_2 d\rfloor}\). The remaining layers form a perfect binary
tree. In the \(j\)-th binary-tree layer, \(1\le j\le \lfloor\log_2 d\rfloor\),
there are \(k/2^j\) pairwise multiplications. Each multiplication uses three
neurons, and hence the width of this layer is at most \(3k/2^j\le 3d/2<2d\).
The total number of neurons is bounded by
\[
3i+2(d-2i)+3\sum_{j=1}^{\lfloor\log_2 d\rfloor}\frac{k}{2^j}
\le 2d+3(k-1)<5d.
\]
Thus the number of neurons is \(\mathcal O(d)\), and the same is true for the
number of nonzero parameters. The depth is the initial compression layer plus
the \(\lfloor\log_2 d\rfloor\) binary-tree layers, namely
\(L\le \lfloor\log_2 d\rfloor+1\). By construction, the output is exactly
\(\prod_{j=1}^d x_j\). The construction is illustrated in
Figure~\ref{fig:prod-square-relu}.
\end{proof}

We next consider ReLU networks. Unlike squared ReLU networks, ReLU networks
cannot represent multiplication exactly, but they can approximate products
efficiently while reminding of the symmetric structure. The next lemma will be useful in estimating the approximation error of neural networks to symmetric sparse grid bases.

\begin{lemma}[Symmetric product approximation by ReLU neural networks]
\label{lem:symmetric-product-relu}
For any \(M,K,d\in\mathbb N_+\) with
\(d\ge2\), there exists a permutation-invariant ReLU neural network
\(\Psi_{M,K,d}:[-1,1]^d\to\mathbb R\) such that
\[
\Psi_{M,K,d}(x_{\pi(1)},\dots,x_{\pi(d)})
=
\Psi_{M,K,d}(x_1,\dots,x_d),
\qquad \forall \pi\in S_d,
\]
and
\[
\left\|
\Psi_{M,K,d}(\vx)-x_1x_2\cdots x_d
\right\|_{W^{1,\infty}([-1,1]^d)}
\le
10(d-1)(M+1)^{-7dK}.
\]
Moreover, if \(x_i=0\) for some \(i\), then
\[
\Psi_{M,K,d}(x_1,\dots,x_d)=0.
\]
The network can be chosen with width
\[
N\le C(M+d)
\]
and depth
\[
L\le 14d(d-1)K+Cd^2,
\]
where \(C>0\) is an absolute constant. In addition, all weights and biases can
be chosen to have magnitude bounded by an absolute constant independent of
\(M\), \(K\), and \(d\).
\end{lemma}

\begin{proof}
Let \(S:[-1,1]^d\to[-1,1]^d\) denote the sorting map, namely
\[
S(\vx)=(S(\vx)_1,\dots,S(\vx)_d),
\qquad
S(\vx)_1\le S(\vx)_2\le \cdots\le S(\vx)_d,\qquad\{S(\vx)_i:~1\leq i\leq d\}=\{x_i:~1\leq i\leq d\}.
\]
We first show that \(S\) can be realized exactly by a ReLU neural network with
controlled size.

For \(i=1,\dots,d-1\), define the adjacent comparator
\(C_i:\mathbb R^d\to\mathbb R^d\) by
\[
C_i(x_1,\dots,x_d)
=
(x_1,\dots,x_{i-1},
\min\{x_i,x_{i+1}\},
\max\{x_i,x_{i+1}\},
x_{i+2},\dots,x_d).
\]
The map \(C_i\) can be realized exactly by a ReLU network. Indeed, for any
\(a,b\in\mathbb R\),
\[
\min\{a,b\}
=
\frac{a+b-|a-b|}{2},
\qquad
\max\{a,b\}
=
\frac{a+b+|a-b|}{2},
\]
and
\[
|a-b|=\sigma(a-b)+\sigma(b-a).
\]
All the other coordinates are passed through unchanged, using the identity
representation
\[
t=\sigma(t)-\sigma(-t).
\]
Therefore each comparator \(C_i\) is an exact ReLU network with width \(Cd\),
depth \(C\), and weights and biases bounded by an absolute constant.

We now compose these adjacent comparators according to the bubble-sort
algorithm. Starting from \(\vx^{(0)}=\vx\), for each pass
\(r=1,\dots,d-1\), we successively apply
\[
C_1,\ C_2,\ \dots,\ C_{d-r}.
\]
Equivalently, the \(r\)-th pass compares the adjacent pairs
\[
(x_1,x_2),\ (x_2,x_3),\ \dots,\ (x_{d-r},x_{d-r+1}),
\]
and replaces each pair by its ordered version
\[
(x_i,x_{i+1})
\mapsto
(\min\{x_i,x_{i+1}\},\max\{x_i,x_{i+1}\}).
\]
After the first pass, the largest entry is moved to the last coordinate; after
the second pass, the second largest entry is moved to the \((d-1)\)-st
coordinate. Continuing this procedure, after \(d-1\) passes the vector is sorted
in nondecreasing order. Hence this composition realizes the sorting map \(S\)
exactly. The total number of comparators is
\[
\sum_{r=1}^{d-1}(d-r)=\frac{d(d-1)}{2}.
\]
Thus \(S\) can be realized by a ReLU network with width \(Cd\), depth \(Cd^2\),
and uniformly bounded weights and biases.

Let \(\Phi_{M,K,d}\) be the product network from
Lemma~\ref{lem:product-relu}. Define
\[
\Psi_{M,K,d}(\vx):=\Phi_{M,K,d}(S(\vx)).
\]
Since sorting is invariant under permutations, for every \(\pi\in S_d\),
\[
S(x_{\pi(1)},\dots,x_{\pi(d)})=S(x_1,\dots,x_d).
\]
Therefore
\[
\Psi_{M,K,d}(x_{\pi(1)},\dots,x_{\pi(d)})
=
\Psi_{M,K,d}(x_1,\dots,x_d),
\]
so \(\Psi_{M,K,d}\) is permutation invariant.

Next, sorting does not change the product:
\[
\prod_{i=1}^d S(\vx)_i
=
\prod_{i=1}^d x_i.
\]
Hence, by Lemma~\ref{lem:product-relu},
\[
\left\|
\Psi_{M,K,d}(\vx)-x_1x_2\cdots x_d
\right\|_{L^\infty([-1,1]^d)}
\le
10(d-1)(M+1)^{-7dK}.
\]

It remains to verify the derivative estimate. The sorting map \(S\) is
piecewise linear and differentiable almost everywhere. At every differentiability
point, \(DS(\vx)\) is a permutation matrix. Therefore, for
\[
E(\vz):=\Phi_{M,K,d}(\vz)-z_1z_2\cdots z_d,
\]
we have, almost everywhere,
\[
\nabla \bigl(E(S(\vx))\bigr)
=
DS(\vx)^\top \nabla E(S(\vx)).
\]
Since \(DS(\vx)\) is a permutation matrix almost everywhere, it does not increase
the \(L^\infty\)-norm of the gradient. Thus
\[
\left\|
\Psi_{M,K,d}(\vx)-x_1x_2\cdots x_d
\right\|_{W^{1,\infty}([-1,1]^d)}
\le
10(d-1)(M+1)^{-7dK}.
\]

Finally, if \(x_i=0\) for some \(i\), then one coordinate of \(S(\vx)\) is also
zero. By the zero-preserving property of \(\Phi_{M,K,d}\), we obtain
\[
\Psi_{M,K,d}(\vx)=\Phi_{M,K,d}(S(\vx))=0.
\]
The width, depth, and parameter bounds follow by composing the sorting network
with the product network \(\Phi_{M,K,d}\). This completes the proof.
\end{proof}

\begin{lemma}\label{lem:prod}
Let $f_{s},g_{s}\in H^{1}([0,1])$ for $s=1,2,\dots,d$ and assume
\[
  \|f_{s}-g_{s}\|_{H^{1}([0,1])}\le\delta,
  \qquad
  \max_{1\le s\le d}\bigl\{\|f_{s}\|_{H^{1}([0,1])},
                           \|g_{s}\|_{H^{1}([0,1])}\bigr\}\le M .
\]
Then
\[
  \Bigl\|
    \prod_{s=1}^{d}f_{s}(x_s)-\prod_{s=1}^{d}g_{s}(x_s)
  \Bigr\|_{H^{1}([0,1]^d)}
  \;\le\;
  d(d+1)\,M^{\,d-1}\,\delta.
\]
\end{lemma}
\begin{proof}
    First, we consider $L^2$-error, we have that \begin{align}
       &\Bigl\|
    \prod_{s=1}^{d}f_{s}(x_s)-\prod_{s=1}^{d}g_{s}(x_s)
  \Bigr\|_{L^{2}([0,1]^d)}\le   \sum_{s=1}^{d}\Bigl\|
    \prod_{t=1}^{s-1}f_{t}(x_t)(f_s(x_s)-g_s(x_s))\prod_{t=s+1}^{d}g_{s}(x_s)
  \Bigr\|_{L^{2}([0,1]^d)}\notag\\\le&  \sum_{s=1}^{d} \prod_{t=1}^{s-1}\|f_{t}(x_t)\|_{L^{2}([0,1])}\prod_{t=s+1}^{d}\|g_{s}(x_s)\|_{L^2([0,1])}\Bigl\|
   f_s(x_s)-g_s(x_s)
  \Bigr\|_{L^{2}([0,1])}\le dM^{d-1}\delta.
    \end{align}For the derivative term, we consider the partial derivative of $s_*$, we have that \begin{align}
        &\Bigl\|
    f'_{s_*}\prod_{s=1,s\neq s_*}^{d}f_{s}(x_s)-g'_{s_*}\prod_{s=1,s\neq s_*}^{d}g_{s}(x_s)
  \Bigr\|_{L^{2}([0,1]^d)}\notag\\\le& \Bigl\|
    (f'_{s_*}-g'_{s_*})\prod_{s=1,s\neq s_*}^{d}f_{s}(x_s)
  \Bigr\|_{L^{2}([0,1]^d)} +\left\|
    g'_{s_*}\left(\prod_{s=1,s\neq s_*}^{d}f_{s}(x_s)-\prod_{s=1,s\neq s_*}^{d}g_{s}(x_s)\right)
  \right\|_{L^{2}([0,1]^d)}\notag\\\le & M^{d-1}\delta+(d-1)M^{d-1}\delta=dM^{d-1}\delta.
    \end{align}Therefore, we have \[
  \Bigl\|
    \prod_{s=1}^{d}f_{s}(x_s)-\prod_{s=1}^{d}g_{s}(x_s)
  \Bigr\|_{H^{1}([0,1]^d)}
  \;\le\;
  d(d+1)\,M^{\,d-1}\,\delta.
\]
\end{proof}

\begin{proposition}[Squared ReLU realization of symmetrized hat bases]
\label{prop:square-relu-sym-basis}
Let \(d\in\mathbb N\), and let \(\vl,\vi\in\mathbb N_+^d\). For any
\(0<\varepsilon\le1\), there exists a symmetric squared ReLU neural network
\(\overline\psi_{\vl,\vi}:\Omega\to\mathbb R\) such that
\begin{equation}
\left\|
\psi_{\vl,\vi}
-
\overline\psi_{\vl,\vi}
\right\|_{H^1(\Omega)}
\le \varepsilon,
\label{eq:square-relu-sym-basis-error}
\end{equation}
where
\[
\psi_{\vl,\vi}(\vx)
:=
\sum_{\tau\in S_d}\phi_{\vl,\vi}(\tau(\vx)).
\]
Moreover,
\[
\operatorname{supp}\overline\psi_{\vl,\vi}
\subset
\bigcup_{\tau\in S_d}\operatorname{supp}\phi_{\vl,\vi}(\tau(\cdot)).
\]
The network can be chosen with width \(N\le C d^3 2^{d-1}\), depth
\(L\le \lfloor \log_2 d\rfloor+2\), and total number of neurons and nonzero
parameters bounded by \(C d^3 2^{d-1}\log_2 d\).

In addition, the parameter magnitudes can be described as follows. The only
parameters depending on \(\varepsilon\) and \(|\vl|_\infty\) occur in the
one-dimensional feature approximation blocks. These parameters can be chosen
with magnitude bounded by
\begin{equation}
C\,2^{|\vl|_\infty}
\left(
\frac{d(d+1)(2M_{\vl,\vi})^{d-1}}{\varepsilon}
\right)^2,
\label{eq:square-relu-param-bound}
\end{equation}
where
\[
M_{\vl,\vi}
:=
\max_{\boldsymbol\delta\in\{\pm1\}^d,\ 1\le s\le d}
\left\|
\sum_{j=1}^d \delta_j\phi_{l_j,i_j}
\right\|_{H^1([0,1])}.
\]
All parameters in the product blocks and in the final Glynn linear combination
are bounded by an absolute constant.
\end{proposition}

\begin{proof}
Let \(\sigma_2(x):=\max\{x,0\}^2\). We first approximate the ReLU function by
squared ReLU networks. For \(\rho>0\), define
\[
S_\rho(x):=\frac{\sigma_2(x)-\sigma_2(x-\rho)}{2\rho}.
\]
Then \(S_\rho\) is a depth-one squared ReLU network of width \(2\). Moreover,
for every fixed \(M>0\),
\begin{equation}
\|S_\rho-\operatorname{ReLU}\|_{H^1([-M,M])}
\le C_M\rho^{1/2}.
\label{eq:Srho-relu-rate}
\end{equation}

Indeed, a direct computation gives
\[
S_\rho'(x)
=
\begin{cases}
1, & x\ge \rho,\\[2mm]
x/\rho, & 0\le x\le \rho,\\[2mm]
0, & x\le 0.
\end{cases}
\]
Hence
\[
S_\rho'(x)-D\operatorname{ReLU}(x)
=
\begin{cases}
x/\rho-1, & 0<x<\rho,\\[2mm]
0, & \text{otherwise},
\end{cases}
\]
and therefore
\[
\|S_\rho'-D\operatorname{ReLU}\|_{L^2([-M,M])}^2
=
\int_0^\rho \left(1-\frac{x}{\rho}\right)^2\,dx
=
\frac{\rho}{3}.
\]
Thus
\[
\|S_\rho'-D\operatorname{ReLU}\|_{L^2([-M,M])}
\le C\rho^{1/2}.
\]

For the \(L^2\)-part, \(S_\rho(x)=\operatorname{ReLU}(x)\) outside
\([0,\rho]\), and on \([0,\rho]\) we have
\[
S_\rho(x)=\frac{x^2}{2\rho},
\qquad
\operatorname{ReLU}(x)=x.
\]
Therefore
\[
\|S_\rho-\operatorname{ReLU}\|_{L^2([-M,M])}^2
=
\int_0^\rho \left(x-\frac{x^2}{2\rho}\right)^2\,dx
\le C\rho^3.
\]
Combining the two estimates gives \eqref{eq:Srho-relu-rate}. Moreover, the only
large parameter in \(S_\rho\) is the output coefficient \((2\rho)^{-1}\), and
therefore the parameters of \(S_\rho\) are bounded by \(C\rho^{-1}\).

Let \(h\) be the standard hat function supported on \([-1,1]\), written as
\(h(x)=\operatorname{ReLU}(x+1)-2\operatorname{ReLU}(x)+\operatorname{ReLU}(x-1)\).
Define
\[
h_\rho(x)
:=
S_\rho(x+1-\rho)-2S_\rho(x)+S_\rho(x-1+\rho).
\]
Then \(h_\rho\to h\) in \(H^1([-M,M])\), and for \(\rho>0\) sufficiently small,
\(\operatorname{supp}h_\rho\subset[-1,1]=\operatorname{supp}h\) as shown in Fig.~\ref{fig:support}. The only
parameters in \(h_\rho\) that may become large are the output coefficients of
the \(S_\rho\)-blocks, and these are bounded by \(C\rho^{-1}\).

\begin{figure}
    \centering
    \includegraphics[width=0.77\linewidth]{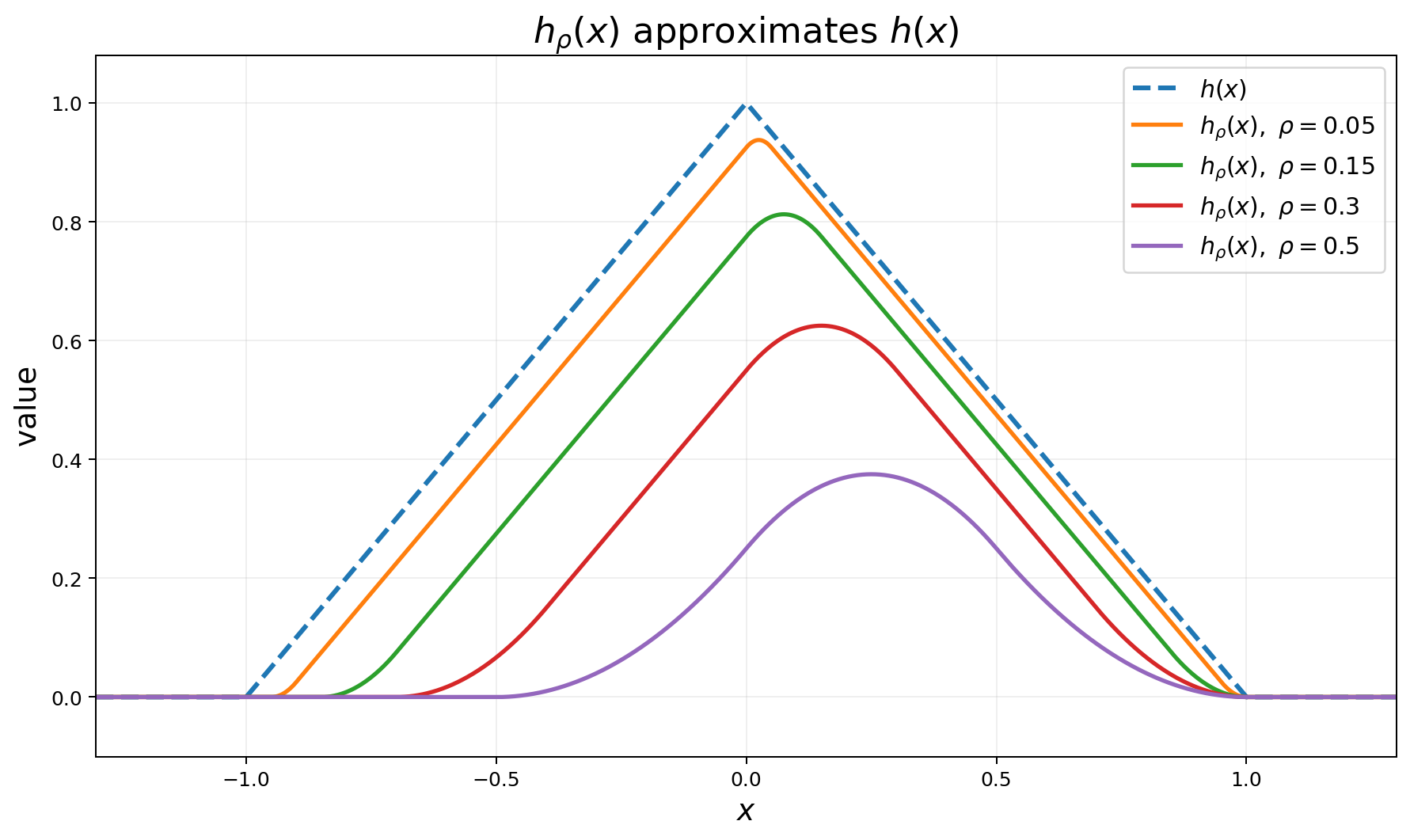}
    \caption{\(\operatorname{supp}h_\rho\subset[-1,1]=\operatorname{supp}h\)}
    \label{fig:support}
\end{figure}

Now consider the scaled hat function \(\phi_{l,i}(x)=h(2^l x-i)\). Its squared
ReLU approximation is \(\phi_{l,i,\rho}(x):=h_\rho(2^l x-i)\). The affine
scaling introduces weights and biases of size at most \(C2^l\), while the
smoothing coefficients coming from \(h_\rho\) are bounded by \(C\rho^{-1}\).
Thus the parameters in this one-dimensional block are bounded by
\(C\max\{2^l,\rho^{-1}\}\). Taking into account the scaling in the \(H^1\)-norm,
for every \(\eta>0\) one can choose \(\rho\) such that
\[
\|\phi_{l,i,\rho}-\phi_{l,i}\|_{H^1([0,1])}\le \eta,
\]
with all parameters in this one-dimensional block bounded by
\(C2^l\eta^{-2}\).

Consequently, for every sign vector
\(\boldsymbol\delta=(\delta_1,\dots,\delta_d)\in\{\pm1\}^d\) and every
\(s=1,\dots,d\), the one-dimensional feature
\(H_{\boldsymbol\delta,s}(x):=\sum_{j=1}^d\delta_j\phi_{l_j,i_j}(x)\) admits a
squared ReLU approximation \(\overline H_{\boldsymbol\delta,s}\) such that
\[
\|H_{\boldsymbol\delta,s}-\overline H_{\boldsymbol\delta,s}\|_{H^1([0,1])}
\le \eta.
\]
This network has depth \(1\), width \(O(d)\), and \(O(d)\) parameters. The only
parameters in this feature block depending on \(\eta\) and \(|\vl|_\infty\) are
those inherited from the one-dimensional hat approximations, and they are
bounded by
\begin{equation}
C\,2^{|\vl|_\infty}\eta^{-2}.
\label{eq:feature-param-bound}
\end{equation}
For \(\eta\) sufficiently small, we also have
\(\|\overline H_{\boldsymbol\delta,s}\|_{H^1([0,1])}\le 2M_{\vl,\vi}\), where
\(M_{\vl,\vi}\) is defined in the statement. Furthermore,
\[
\operatorname{supp}\overline H_{\boldsymbol\delta,s}
\subset
\bigcup_{j=1}^d\operatorname{supp}\phi_{l_j,i_j}.
\]

By Glynn's formula, Lemma~\ref{lem:sym-product-glynn},
\[
\psi_{\vl,\vi}(\vx)
=
2^{1-d}
\sum_{\substack{\boldsymbol\delta\in\{\pm1\}^d\\ \delta_1=1}}
\left(\prod_{k=1}^d\delta_k\right)
\prod_{s=1}^d H_{\boldsymbol\delta,s}(x_s).
\]
For each \(\boldsymbol\delta\), define
\(\overline G_{\boldsymbol\delta}(\vx):=\prod_{s=1}^d
\overline H_{\boldsymbol\delta,s}(x_s)\). By
Lemma~\ref{lem:product-square-relu}, this product can be realized exactly by a
squared ReLU network of depth \(\lfloor\log_2 d\rfloor+1\), width at most
\(2d\), and \(O(d)\) parameters. Importantly, all parameters in this product
block are bounded by an absolute constant and are independent of
\(\varepsilon\), \(\eta\), and \(|\vl|_\infty\).

We now estimate the error. By Lemma~\ref{lem:prod}, applied with
\(f_s=H_{\boldsymbol\delta,s}\), \(g_s=\overline H_{\boldsymbol\delta,s}\), and
\(M=2M_{\vl,\vi}\), we obtain
\[
\left\|
\prod_{s=1}^d H_{\boldsymbol\delta,s}(x_s)
-
\prod_{s=1}^d \overline H_{\boldsymbol\delta,s}(x_s)
\right\|_{H^1(\Omega)}
\le
d(d+1)(2M_{\vl,\vi})^{d-1}\eta.
\]
Therefore, after summing over the \(2^{d-1}\) sign vectors in Glynn's formula
and using the prefactor \(2^{1-d}\), the same bound holds for the symmetrized
basis:
\[
\left\|
\psi_{\vl,\vi}
-
2^{1-d}
\sum_{\substack{\boldsymbol\delta\in\{\pm1\}^d\\ \delta_1=1}}
\left(\prod_{k=1}^d\delta_k\right)
\overline G_{\boldsymbol\delta}
\right\|_{H^1(\Omega)}
\le
d(d+1)(2M_{\vl,\vi})^{d-1}\eta.
\]
Choose \(\eta:=\varepsilon/[d(d+1)(2M_{\vl,\vi})^{d-1}]\). Then the
approximation error is at most \(\varepsilon\). Define
\[
\overline\psi_{\vl,\vi}(\vx)
:=
2^{1-d}
\sum_{\substack{\boldsymbol\delta\in\{\pm1\}^d\\ \delta_1=1}}
\left(\prod_{k=1}^d\delta_k\right)
\overline G_{\boldsymbol\delta}(\vx).
\]

The support property follows from
\(\operatorname{supp}\overline H_{\boldsymbol\delta,s}
\subset\bigcup_{j=1}^d\operatorname{supp}\phi_{l_j,i_j}\) and from the product
construction. Hence
\[
\operatorname{supp}\overline\psi_{\vl,\vi}
\subset
\bigcup_{\tau\in S_d}\operatorname{supp}\phi_{\vl,\vi}(\tau(\cdot)).
\]

It remains to estimate the architecture. For each fixed
\(\boldsymbol\delta\), constructing the \(d\) feature functions
\(\overline H_{\boldsymbol\delta,s}\), \(s=1,\dots,d\), requires width
\(O(d^2)\) and \(O(d^2)\) parameters. The product realization contributes width
\(O(d)\), depth \(\lfloor\log_2 d\rfloor+1\), and \(O(d)\) parameters. Since
there are \(2^{d-1}\) sign vectors, the full network can be chosen with width
\(N\le Cd^3 2^{d-1}\), depth \(L\le\lfloor\log_2 d\rfloor+2\), and total number
of neurons and nonzero parameters bounded by \(Cd^3 2^{d-1}\log_2 d\).

Finally, we record the parameter magnitudes. By \eqref{eq:feature-param-bound}
and the choice of \(\eta\), the parameters in the one-dimensional feature
approximation blocks are bounded by
\[
C\,2^{|\vl|_\infty}\eta^{-2}
=
C\,2^{|\vl|_\infty}
\left(
\frac{d(d+1)(2M_{\vl,\vi})^{d-1}}{\varepsilon}
\right)^2.
\]
All parameters in the product blocks are bounded by an absolute constant, and
the final Glynn linear combination only uses coefficients of magnitude
\(2^{1-d}\). Hence the displayed bound is only needed for the innermost
one-dimensional feature approximation blocks; all outer blocks have uniformly
bounded parameters. This
completes the proof.
\end{proof}

\begin{proposition}[Symmetric ReLU approximation of symmetrized hat basis functions]
\label{prop:sym-basis-relu}
For any \(\vl,\vi\in \mathbb N^d\) and any \(0<\varepsilon\le 1\), there exists
a symmetric ReLU neural network
\(\Phi_{\vl,\vi,\varepsilon}:\Omega\to\mathbb R\) such that
\[
\left\|
\sum_{\tau\in S_d}\phi_{\vl,\vi}(\tau(\vx))
-
\Phi_{\vl,\vi,\varepsilon}(\vx)
\right\|_{W^{1,\infty}(\Omega)}
\le \varepsilon .
\]
Moreover,
\(\Phi_{\vl,\vi,\varepsilon}(\vx)=0\) for all \(\vx\in\partial\Omega\).

The network can be chosen with width
\[
N\le 3d^2+C2^{d-1}d
\]
and depth
\[
L\le 1+14d(d-1)L_\varepsilon+Cd^2,
\]
where
\[
L_\varepsilon
:=
\left\lceil
\frac{
\log(\varepsilon^{-1})+(d+1)\log d+|\vl|_\infty\log 2+\log 20
}{
7d\log(2d+1)
}
\right\rceil .
\]
The total number of nonzero parameters is bounded by
\(C(d^3+2^d d^2L_\varepsilon)\).

\end{proposition}

\begin{proof}
By Lemma~\ref{lem:sym-product-glynn},
\begin{align}
\sum_{\tau\in S_d}\phi_{\vl,\vi}(\tau(\vx))
=
\sum_{\tau\in S_d}\prod_{j=1}^d \phi_{l_j,i_j}(x_{\tau(j)})  =
2^{1-d}
\sum_{\substack{\boldsymbol{\delta}\in\{\pm1\}^d\\ \delta_1=1}}
\left(\prod_{j=1}^d \delta_j\right)
\prod_{i=1}^d
\left(\sum_{j=1}^d \delta_j\,\phi_{l_j,i_j}(x_i)\right).
\label{eq:glynn-basis-proof}
\end{align}
For each admissible sign vector \(\boldsymbol{\delta}\), define
\(H_{\boldsymbol{\delta},i}(\vx)
:=\sum_{j=1}^d\delta_j\phi_{l_j,i_j}(x_i)\), \(i=1,\dots,d\), and set
\(H_{\boldsymbol{\delta}}(\vx)
:=(H_{\boldsymbol{\delta},1}(\vx),\dots,H_{\boldsymbol{\delta},d}(\vx))\).
Since each one-dimensional hat function has the form
\(\phi_{l,i}(x)=\phi((x-i2^{-l})/2^{-l})\), all features
\(\phi_{l_j,i_j}(x_i)\), \(1\le i,j\le d\), can be realized exactly by a ReLU
network with width \(3d^2\), depth \(1\), and parameter magnitudes bounded by
\(C2^{|\vl|_\infty}\). The linear combinations defining
\(H_{\boldsymbol{\delta},i}\) only use coefficients \(\pm1\). Hence
\[
\|H_{\boldsymbol{\delta}}\|_{W^{1,\infty}(\Omega;\mathbb R^d)}
\le 1+d2^{|\vl|_\infty},
\qquad
\|H_{\boldsymbol{\delta}}\|_{L^\infty(\Omega;\mathbb R^d)}
\le d .
\]

Next, we apply the symmetric product approximation
Lemma~\ref{lem:symmetric-product-relu}, after rescaling from \([-1,1]^d\) to
\((-d,d)^d\). For every \(L\in\mathbb N_+\), there exists a
permutation-invariant ReLU network
\(P_{\boldsymbol{\delta},L}:(-d,d)^d\to\mathbb R\) such that
\[
P_{\boldsymbol{\delta},L}(z_{\pi(1)},\dots,z_{\pi(d)})
=
P_{\boldsymbol{\delta},L}(z_1,\dots,z_d),
\qquad \forall \pi\in S_d,
\]
and
\[
\|P_{\boldsymbol{\delta},L}(\vz)-z_1\cdots z_d\|_{W^{1,\infty}((-d,d)^d)}
\le
10(d-1)d^d(2d+1)^{-7dL}.
\]
Moreover, \(P_{\boldsymbol{\delta},L}(\vz)=0\) whenever \(z_i=0\) for some
\(i\). The product block has width \(Cd\), depth
\(14d(d-1)L+Cd^2\), and \(O(d^2L)\) nonzero parameters.

Define
\(\widetilde G_{\boldsymbol{\delta},L}(\vx)
:=P_{\boldsymbol{\delta},L}(H_{\boldsymbol{\delta}}(\vx))\). By the chain rule
and the bound on \(H_{\boldsymbol{\delta}}\),
\[
\left\|
\widetilde G_{\boldsymbol{\delta},L}
-
\prod_{i=1}^d H_{\boldsymbol{\delta},i}
\right\|_{W^{1,\infty}(\Omega)}
\le
10(d-1)d^d(1+d2^{|\vl|_\infty})(2d+1)^{-7dL}.
\]
Taking \(L=L_\varepsilon\), the right-hand side is at most \(\varepsilon\).

We now show that \(\widetilde G_{\boldsymbol{\delta},L}\) is permutation
invariant in \(\vx\). For any \(\pi\in S_d\),
\(H_{\boldsymbol{\delta}}(x_{\pi(1)},\dots,x_{\pi(d)})\) is just a permutation
of the coordinates of \(H_{\boldsymbol{\delta}}(\vx)\). Since
\(P_{\boldsymbol{\delta},L}\) is permutation invariant, it follows that
\[
\widetilde G_{\boldsymbol{\delta},L}(x_{\pi(1)},\dots,x_{\pi(d)})
=
\widetilde G_{\boldsymbol{\delta},L}(\vx).
\]

Finally define
\[
\Phi_{\vl,\vi,\varepsilon}(\vx)
:=
2^{1-d}
\sum_{\substack{\boldsymbol{\delta}\in\{\pm1\}^d\\ \delta_1=1}}
\left(\prod_{j=1}^d\delta_j\right)
\widetilde G_{\boldsymbol{\delta},L_\varepsilon}(\vx).
\]
Since each \(\widetilde G_{\boldsymbol{\delta},L_\varepsilon}\) is permutation
invariant, their linear combination
\(\Phi_{\vl,\vi,\varepsilon}\) is also permutation invariant. Using
\eqref{eq:glynn-basis-proof} and the preceding product approximation estimate,
the prefactor \(2^{1-d}\) cancels the number \(2^{d-1}\) of admissible sign
vectors, and hence
\[
\left\|
\sum_{\tau\in S_d}\phi_{\vl,\vi}(\tau(\vx))
-
\Phi_{\vl,\vi,\varepsilon}(\vx)
\right\|_{W^{1,\infty}(\Omega)}
\le \varepsilon .
\]

If \(\vx\in\partial\Omega\), then some coordinate \(x_i\in\{0,1\}\). Every
one-dimensional hat function vanishes at \(0\) and \(1\), so
\(H_{\boldsymbol{\delta},i}(\vx)=0\). By the zero-preserving property of the
product block,
\(\widetilde G_{\boldsymbol{\delta},L_\varepsilon}(\vx)=0\) for every
\(\boldsymbol{\delta}\). Therefore
\(\Phi_{\vl,\vi,\varepsilon}(\vx)=0\) on \(\partial\Omega\).

It remains to record the architecture. The feature layer realizing all
\(\phi_{l_j,i_j}(x_i)\) has width \(3d^2\), depth \(1\), and \(O(d^3)\)
nonzero parameters. For each of the \(2^{d-1}\) admissible sign vectors, the
symmetric product block has width \(Cd\), depth
\(14d(d-1)L_\varepsilon+Cd^2\), and \(O(d^2L_\varepsilon)\) nonzero parameters.
Placing these blocks in parallel gives the stated width, depth, and parameter
bounds.
\end{proof}

\begin{remark}
    The parameter magnitude statement of neural networks in Proposition \ref{prop:sym-basis-relu} is shown as follows. The
hat-function feature blocks contain affine scalings of size \(2^{l_j}\), hence
their parameters are bounded by \(C2^{|\vl|_\infty}\). The internal parameters
of the product blocks and of the sorting networks are uniformly bounded. The
only coefficients that may grow like \(d^d\) are the final output-layer
coefficients inside the scaled product approximation blocks. The final Glynn
linear combination uses coefficients \(2^{1-d}\prod_j\delta_j\), whose
magnitudes are at most \(1\). This completes the proof.
\end{remark}

\subsection{Proof of Theorem \ref{thm:main-sym-korobov neural network}}

\begin{proof}[Proof of Theorem~\ref{thm:main-sym-korobov neural network}]
By Corollary~\ref{better}, the symmetric sparse-grid approximation uses at most
\(C_s2^n\) active symmetric basis functions and satisfies
\[
\|f-f_n^{(2)}\|_E
\le
\frac{48C_s d}{5\sqrt3}
\left(\frac{5}{12}\right)^d
(C_s2^n)^{-1}|f|_{2,2}.
\]
Hence
\begin{equation}
\|f-f_n^{(2)}\|_E
\le
C d
\left(\frac{5}{12}\right)^d
2^{-n}|f|_{2,2}.
\label{eq:sg-error-main}
\end{equation}
Moreover, the coefficients in the expansion satisfy
\begin{align}
\sum_{\vl,\vi}|v_{\vl,\vi}|
&=
\sum_{\vl,\vi}
\left|
\frac{|\{\phi_{\vl,\vi}(\tau(\vx)):\tau\in S_d\}|}{d!}
\int_{\Omega}\phi_{\vl,\vi}(\vx)\,
\frac{\partial^{2d}f}{\partial x_1^2\cdots \partial x_d^2}\,d\vx
\right|
\notag\\
&\le
\sum_{\vl,\vi}
\left|
\int_{\Omega}\phi_{\vl,\vi}(\vx)\,
\frac{\partial^{2d}f}{\partial x_1^2\cdots \partial x_d^2}\,d\vx
\right|
\notag\\
&\le
\sum_{\vl,\vi}
2^{-d}\left(\frac{2}{3}\right)^{d/2}2^{-\frac32|\vl|_1}
\,|f|_{2,2;\operatorname{supp}(\phi_{\vl,\vi})}
\notag\\
&\le
\sum_{\vl}
2^{-d}\left(\frac{2}{3}\right)^{d/2}2^{-\frac32|\vl|_1}
\,|f|_{2,2}
=:C_0\,|f|_{2,2}.
\label{eq:coeff-l1-main}
\end{align}

We first prove the squared ReLU result. For each active symmetric basis
function, apply Proposition~\ref{prop:square-relu-sym-basis} with accuracy
\(\eta:=(5/12)^d(C_s2^n)^{-1}\). Thus each active basis function admits a
permutation-invariant squared ReLU approximation with \(H^1\)-error at most
\(\eta\). Taking the linear combination of these basis-level networks with
coefficients \(\bar v_{\vl,\vi}\), we obtain a permutation-invariant squared
ReLU network \(\Phi_{\mathrm{ReLU}_2}\). By \eqref{eq:coeff-l1-main},
\[
\|f_n^{(2)}-\Phi_{\mathrm{ReLU}_2}\|_E
\le
\eta\sum_{\vl,\vi}|\bar v_{\vl,\vi}|
\le
C_0
\left(\frac{5}{12}\right)^d
(C_s2^n)^{-1}|f|_{2,2}.
\]
Combining this with \eqref{eq:sg-error-main} yields
\begin{equation}
\|f-\Phi_{\mathrm{ReLU}_2}\|_E
\le
C d
\left(\frac{5}{12}\right)^d
2^{-n}|f|_{2,2}.
\label{eq:main-square-relu-error}
\end{equation}
The boundary condition follows from the boundary-vanishing property of the
basis-level squared ReLU approximations.

By Proposition~\ref{prop:square-relu-sym-basis}, each active basis approximation
has width \(Cd^3 2^{d-1}\), depth \(\lfloor\log_2d\rfloor+2\), and
\(Cd^3 2^{d-1}\log_2d\) nonzero parameters. Since there are at most
\(C_s2^n\) active basis functions, placing the subnetworks in parallel and
taking the final linear combination gives
\[
N_2\le C_s d^3 2^{n+d-1},\qquad
L_2\le \lfloor\log_2 d\rfloor+2,\qquad
\mathcal M_2\le C_s d^3 2^{n+d-1}\log_2 d.
\]
Therefore \(2^{-n}\le C d^3 2^{d-1}(\log_2 d)\mathcal M_2^{-1}\). Substituting
this into \eqref{eq:main-square-relu-error} gives
\[
\|f-\Phi_{\mathrm{ReLU}_2}\|_E
\le
C d^4
\left(\frac{5}{12}\right)^d
2^{d-1}
(\log_2 d)\mathcal M_2^{-1}|f|_{2,2}.
\]
Since \(2^{d-1}(5/12)^d\le (5/6)^d\), we obtain
\[
\|f-\Phi_{\mathrm{ReLU}_2}\|_E
\le
C d^4
\left(\frac{5}{6}\right)^d
(\log_2 d)
\mathcal M_2^{-1}|f|_{2,2}.
\]

We next prove the ReLU result. For each active symmetric basis function, apply
Proposition~\ref{prop:sym-basis-relu} with the same accuracy
\(\eta=(5/12)^d(C_s2^n)^{-1}\). Each active basis function therefore admits a
permutation-invariant ReLU approximation with \(H^1\)-error at most \(\eta\).
Taking the linear combination of these ReLU basis networks with coefficients
\(\bar v_{\vl,\vi}\), we obtain a permutation-invariant ReLU network
\(\Phi_{\mathrm{ReLU}}\). By \eqref{eq:coeff-l1-main},
\[
\|f_n^{(2)}-\Phi_{\mathrm{ReLU}}\|_E
\le
C_0
\left(\frac{5}{12}\right)^d
(C_s2^n)^{-1}|f|_{2,2}.
\]
Combining this with \eqref{eq:sg-error-main} gives
\[
\|f-\Phi_{\mathrm{ReLU}}\|_E
\le
C d
\left(\frac{5}{12}\right)^d
2^{-n}|f|_{2,2}.
\]
The boundary condition follows from the boundary-vanishing property in
Proposition~\ref{prop:sym-basis-relu}.

By Proposition~\ref{prop:sym-basis-relu}, each active ReLU basis network has
width at most \(3d^2+C2^{d-1}d\), depth at most
\(1+14d(d-1)L_\eta+Cd^2\), and at most
\(C(d^3+2^d d^2L_\eta)\) nonzero parameters. Since there are at most
\(C_s2^n\) active basis functions, the full ReLU network satisfies
\[
N_1
\le
C_s2^n(3d^2+C2^{d-1}d)
\le
C_s d^2 2^{n+d}.
\]
Moreover, since \(\eta^{-1}=(12/5)^d C_s2^n\), we have
\(\log(\eta^{-1})=n\log2+d\log(12/5)+\log C_s\). Hence
\[
L_\eta
\le
C\left(
1+\frac{n}{d\log(2d+1)}
\right),
\]
and therefore
\[
L_1
\le
1+14d(d-1)L_\eta+Cd^2
\le
C_s d^2(1+n).
\]
Similarly,
\[
\mathcal M_1
\le
C_s2^n\bigl(d^3+2^d d^2L_\eta\bigr)
\le
C_s d^2 2^{n+d}
\left(
1+\frac{n}{d\log(2d+1)}
\right).
\]

It remains to rewrite the approximation error in terms of \(\mathcal M_1\).
From the preceding bound,
\[
2^{-n}
\le
C d^2 2^d
\left(
1+\frac{n}{d\log(2d+1)}
\right)\mathcal M_1^{-1}.
\]
Combining this with
\[
\|f-\Phi_{\mathrm{ReLU}}\|_E
\le
C d\left(\frac{5}{12}\right)^d2^{-n}|f|_{2,2}
\]
gives
\[
\|f-\Phi_{\mathrm{ReLU}}\|_E
\le
C d^3
\left(\frac56\right)^d
\left(
1+\frac{n}{d\log(2d+1)}
\right)
\mathcal M_1^{-1}|f|_{2,2}.
\]
Since \(2^n\le C\mathcal M_1\), we have \(n\le C\log(2\mathcal M_1)\). Thus
\[
\|f-\Phi_{\mathrm{ReLU}}\|_E
\le
C d^3
\left(\frac56\right)^d
\frac{\log(2\mathcal M_1)}{\mathcal M_1}
|f|_{2,2}.
\]

The parameter-magnitude statements follow directly from
Proposition~\ref{prop:square-relu-sym-basis} and
Proposition~\ref{prop:sym-basis-relu}, together with the coefficient estimate
\eqref{eq:coeff-l1-main}. This completes the proof.
\end{proof}

We remark that the approximation rate in Theorem~\ref{thm:main-sym-korobov neural network} is optimal among all continuous approximation schemes.  
Indeed, the space $H^{2}_{0}([0,1])$ is a one–dimensional special case of the Korobov space $X^{2,2}(\Omega)$.  
For this space it is known \cite{devore1989optimal} that no continuous approximator using $m$ free parameters can outperform the rate $\mathcal{O}(m^{-1})$ measured by the $H^1$-norm.  Formally, consider a subset \(X\) of a Banach space, a set of neural networks with \(N\) parameters, and an approximation scheme \(G : X \to \mathbb{R}^N\) that, given an input \(f \in X\), gives as output the parameters \(\boldsymbol{\theta}_f = G(f)\) of the neural network approximating \(f\). If \(G\) is continuous, then we call it a continuous function approximator.  
For such schemes the lower bound of \cite{devore1989optimal} gives the rate $m^{-1}$ in $f\in H^{2}_0([0,1])$.  
Since Theorem~\ref{thm:main-sym-korobov neural network} achieves the matching upper bound $\mathcal{O}(m^{-1})$ measured by the $H^1$-norm, its rate cannot be improved within the framework of continuous approximators.

For the generalization error analysis in the next section, we restate
Theorem~\ref{thm:main-sym-korobov neural network} in the following equivalent form.

\begin{corollary}[Decomposition form for generalization analysis]
\label{cor:gen-error-form}
Let \(f\in X_{\mathrm{sym}}^{2,2}(\Omega)\) with \(\Omega=[0,1]^d\), and let
\(n\in\mathbb N_+\). Then there exists \(m\) such that the following two
statements hold.

\emph{(i) Squared ReLU case.}
There exists a permutation-invariant squared ReLU neural network
\(\Phi_{\mathrm{ReLU}_2}(\vx)=\sum_{i=1}^{m}s_i^{(2)}(\vx)\), where each
summand \(s_i^{(2)}\) is a permutation-invariant squared ReLU subnetwork with
width \(Cd^2\) and depth \(\lfloor\log_2 d\rfloor+2\). Moreover,
\(\Phi_{\mathrm{ReLU}_2}|_{\partial\Omega}=0\), and
\begin{equation}
\|f-\Phi_{\mathrm{ReLU}_2}\|_E
\le
C d\left(\frac{5}{6}\right)^d
m^{-1}|f|_{2,2}.
\label{eq:gen-form-square-relu}
\end{equation}

\emph{(ii) ReLU case.}
There exists a permutation-invariant ReLU neural network
\(\Phi_{\mathrm{ReLU}}(\vx)=\sum_{i=1}^{m}s_i^{(1)}(\vx)\), where
\(m\le C2^{n+d-1}\), and each summand \(s_i^{(1)}\) is a
permutation-invariant ReLU subnetwork with width \(N_1\le Cd^2\) and depth
\(L_1\le Cd^2(\log m+d\log d)\). Moreover,
\(\Phi_{\mathrm{ReLU}}|_{\partial\Omega}=0\), and
\begin{equation}
\|f-\Phi_{\mathrm{ReLU}}\|_E
\le
C d\left(\frac{5}{6}\right)^d
m^{-1}|f|_{2,2}.
\label{eq:gen-form-relu}
\end{equation}
\end{corollary}


%

%

\section{Proof of Theorem \ref{general thm}}\label{sec:thm2}

We now proceed to the proof of the generalization error bound stated in
Theorem~\ref{general thm}. As shown in
Theorem~\ref{thm:main-sym-korobov neural network}, our approximation results
involve two types of neural networks: ReLU networks and squared ReLU networks.
To make the analysis more convenient and better adapted to both constructions,
we control the estimation error through the pseudo-dimension of the derivatives
of the network class, following the approach of \cite{yang2023nearly}. This
provides a unified way to analyze the generalization behavior of both network
classes constructed in Theorem~\ref{thm:main-sym-korobov neural network}.

This point is important because our generalization estimate incorporates
derivative information. Existing neural-network results for Korobov-type
function classes, such as
\cite{suzuki2018adaptivity,montanelli2019new,yang2024near,mao2022approximation,liu2025approximation,li2025higher,li2025some,yang2026approximation}, either do not include derivative
information or focus only on the approximation error. Incorporating gradient
information substantially changes the estimation-error analysis. In particular,
the derivative of ReLU are not globally Lipschitz continuous, and hence the
covering number of the corresponding derivative class may be infinite on the
whole domain under the standard sup-norm metric. To overcome this difficulty,
we use a uniform empirical covering number argument and bound the resulting
covering number by the pseudo-dimension of the derivative network class.
To the best of our knowledge, this pseudo-dimension-based treatment of
derivative classes has not previously been applied to generalization analysis
for Korobov-space neural-network approximation.

We start by recalling the relevant notion of covering numbers.
\begin{definition}[covering number \cite{anthony1999neural}]
				Let $(V,\|\cdot\|)$ be a normed space, and $\Theta\subset V$. $\{V_1,V_2,\ldots,V_n\}$ is an $\varepsilon$-covering of $\Theta$ if $\Theta\subset \cup_{i=1}^nB_{\varepsilon,\|\cdot\|}(V_i)$. The covering number $\fN(\varepsilon,\Theta,\|\cdot\|)$ is defined as \(\fN(\varepsilon,\Theta,\|\cdot\|):=\min \{n: \exists \varepsilon \text {-covering over } \Theta \text { of size } n\} \text {. }\)
			\end{definition}
			
			\begin{definition}[Uniform covering number \cite{anthony1999neural}]\label{uniform}
				{Suppose the $\fF$ is a class of functions from $\fX$ to $\sR$.} Given $n$ samples $\vZ_n=(z_1,\ldots,z_n)\in\fX^n$, define \[\fF|_{\vZ_n}=\{(u(z_1),\ldots,u(z_n)):u\in\fF\}.\]The uniform covering number $\fN(\varepsilon,\fF,n)$ is defined as \[\fN(\varepsilon,\fF,n)=\max_{\vZ_n\in\fX^n}\fN\left(\varepsilon, \fF|_{\vZ_n},\|\cdot\|_{\infty}\right),\]where $\fN\left(\varepsilon, \fF|_{\vZ_n},\|\cdot\|_{\infty}\right)$ denotes the $\varepsilon$-covering number of $\fF|_{\vZ_n}$ w.r.t the $l^\infty$-norm on $\vZ_n$ defined as $\|f\|_{\infty}=\sup_{\vz_i\in\vZ_n}|f(z_i)|$.
			\end{definition}


\begin{proposition}\label{prop:connect}
    Let $M \in \sN$, and assume that $\|\nabla f_\rho\|_{L^{\infty}(\Omega)}\le L$ and $\fY\in[-L,L]^d$ are almost surely for some $L \geq 1$. Then for $\kappa=1,2$, we have that \begin{align*}
        &\mathbb{E}\|f_{\fS,\fF_{m,L,C_*}^\kappa}-f_\rho\|_{E}^2\notag\\\le& \sum_{k=1}^d\frac{5136 L^4}{M}\left[\log\left(14 \fN\left(\frac{ 1}{80 LM}, \partial_k\fF_{m,L,C_*}^\kappa,M\right) \right)+1\right]+2\inf_{f\in\fF_{m,L,C_*}^\kappa}\|f-f_\rho\|^2_{E}.
    \end{align*}
\end{proposition}

\begin{proof}
    First, we divide \(\|f_{\fS,\fF_{m,L,C_*}^\kappa}-f_\rho\|_{E}^2\) into two parts \begin{align}&\|f_{\fS,\fF_{m,L,C_*}^\kappa}-f_\rho\|_{E}^2\notag\\=&\underbrace{\|f_{\fS,\fF_{m,L,C_*}^\kappa}-f_\rho\|_{E}^2-2(\fE_{\fS}(f_{\fS,\fF_{m,L,C_*}^\kappa})-\fE_{\fS}(f_\rho))}_{\mathbf{A}_0}+2(\fE_{\fS}(f_{\fS,\fF_{m,L,C_*}^\kappa})-\fE_{\fS}(f_\rho)).\end{align}For the second part, we have that \[\begin{aligned}
        &2\mathbb{E}(\fE_{\fS}(f_{\fS,\fF_{m,L,C_*}^\kappa})-\fE_{\fS}(f_\rho))\\=& 2\mathbb{E}\inf_{f\in\fF_{m,L,C_*}^\kappa}(\fE_{\fS}(f)-\fE_{\fS}(f_\rho))\le 2\inf_{f\in\fF_{m,L,C_*}^\kappa}\mathbb{E}(\fE_{\fS}(f)-\fE_{\fS}(f_\rho))\\=&\frac{2}{M}\inf_{f\in\fF_{m,L,C_*}^\kappa}\mathbb{E}\left(\sum_{j=1}^M|\nabla f(\vx_i)-\nabla f_\rho(\vx_i)|^2+2\sum_{j=1}^M(\nabla f(\vx_i)-\nabla f_\rho(\vx_i))(\nabla f_\rho(\vx_i)-y_i)\right)\\=&2 \inf _{f \in \fF_{m,L,C_*}^\kappa}\left\|f-f_\rho\right\|_E^2,\notag
    \end{aligned}\]where the last equality is use to   \[\mathbb{E}_{\vx_i}(\nabla f(\vx_i)-\nabla f_\rho(\vx_i))\mathbb{E}_{y_i}[(\nabla f_\rho(\vx_i)-y_i)|\vx_i]=\mathbb{E}_{\vx_i}(\nabla f(\vx_i)-\nabla f_\rho(\vx_i))\cdot 0=0.\]As for $\mathbf{A}_0$, we have \[\mathbf{A}_0=\sum_{k=1}^d\mathbf{A}_{0,k}\]where \[\mathbf{A}_{0,k}:=\|\partial_kf_{\fS,\fF_{m,L,C_*}^\kappa}-\partial_kf_\rho\|_{L^2(\Omega)}^2-2\left[\frac1M\sum_{j=1}^{M}\bigl(\partial_kf_{\fS,\fF_{m,L,C_*}^\kappa}-y_{j,k}\bigr)^{2}-\frac1M\sum_{j=1}^{M}\bigl(\partial_kf_\rho-y_{j,k}\bigr)^{2}\right],\]and $\vy_j=(y_{j,1},\ldots,y_{j,d})$. 
    
    Then for each $\mathbf{A}_{0,k}$, we have that
    \begin{align}
    &\mathbb{P}(\mathbf{A}_{0,k}\ge \epsilon)\notag\\=&\mathbb{P}\Big( 2\|\partial_kf_{\fS,\fF_{m,L,C_*}^\kappa}-\partial_kf_\rho\|_{L^2(\Omega)}^2-2\left[\frac1M\sum_{j=1}^{M}\bigl(\partial_kf_{\fS,\fF_{m,L,C_*}^\kappa}-y_{j,k}\bigr)^{2}-\frac1M\sum_{j=1}^{M}\bigl(\partial_kf_\rho-y_{j,k}\bigr)^{2}\right]\notag\\&\ge\epsilon+\|\partial_kf_{\fS,\fF_{m,L,C_*}^\kappa}-\partial_kf_\rho\|_{L^2(\Omega)}^2\Big)\notag\\=&\mathbb{P}\Big( \|\partial_kf_{\fS,\fF_{m,L,C_*}^\kappa}-\partial_kf_\rho\|_{L^2(\Omega)}^2-\left[\frac1M\sum_{j=1}^{M}\bigl(\partial_kf_{\fS,\fF_{m,L,C_*}^\kappa}-y_{j,k}\bigr)^{2}-\frac1M\sum_{j=1}^{M}\bigl(\partial_kf_\rho-y_{j,k}\bigr)^{2}\right]\notag\\&\ge\frac{1}{2}\left(\frac{1}{2}\epsilon+\frac{1}{2}\epsilon+\|\partial_kf_{\fS,\fF_{m,L,C_*}^\kappa}-\partial_kf_\rho\|_{L^2(\Omega)}^2\right)\Big)\notag\\\le &14 \fN\left(\frac{ \epsilon}{80 L}, \partial_k\fF_{m,L,C_*}^\kappa,M\right) \exp \left(-\frac{\epsilon M}{5136 L^4}\right),\notag
\end{align}where the last inequality is due to \cite[Theorem 11.4]{gyorfi2002distribution}.

Therefore, we have that \[\begin{aligned}
    \mathbb{E} \mathbf{A}_{0,k}&\le \int_{0}^\infty \mathbb{P}(\mathbf{A}_{0,k}\ge t)\,\mathrm{d} t\le\epsilon+\int_{\epsilon}^\infty \mathbb{P}(\mathbf{A}_{0,k}\ge t)\,\mathrm{d} y\\&\le\epsilon+\int_{\epsilon}^\infty 14 \fN\left(\frac{ \epsilon}{80 L}, \partial_k\fF_{m,L,C_*}^\kappa,M\right) \exp \left(-\frac{t M}{5136 L^4}\right)\,\mathrm{d} t\notag
\end{aligned}\]

By the direct calculation, we have \begin{align}
    &\int_{\epsilon}^\infty 14 \fN\left(\frac{ \epsilon}{80 L}, \partial_k\fF_{m,L,C_*}^\kappa,M\right) \exp \left(-\frac{t M}{5136 L^4}\right)\,\mathrm{d} t\notag\\\le& 14 \fN\left(\frac{ \epsilon}{80 L}, \partial_k\fF_{m,L,C_*}^\kappa,M\right) \frac{5136 L^4}{M}\exp \left(-\frac{\epsilon M}{5136 L^4}\right).\notag
\end{align} Set \[\epsilon=\frac{5136 L^4}{M}\log\left(14 \fN\left(\frac{ 1}{80 LM}, \partial_k\fF_{m,L,C_*}^\kappa,M\right) \right)\ge \frac{1}{M}\] and we have \[\mathbb{E} \mathbf{A}_{0,k}\le \frac{5136 L^4}{M}\left[\log\left(14 \fN\left(\frac{ 1}{80 LM}, \partial_k\fF_{m,L,C_*}^\kappa,M\right) \right)+1\right]\]

Hence we have \begin{align*}
        &\mathbb{E}\|f_{\fS,\fF_{m,L,C_*}^\kappa}-f_\rho\|_{E}^2\notag\\\le& \sum_{k=1}^d\frac{5136 L^4}{M}\left[\log\left(14 \fN\left(\frac{ 1}{80 LM}, \partial_k\fF_{m,L,C_*}^\kappa,M\right) \right)+1\right]+2\inf_{f\in\fF_{m,L,C_*}^\kappa}\|f-f_\rho\|^2_{E}.\notag
    \end{align*}
\end{proof}

We bound the covering numbers via the pseudo-dimension, defined below.
\begin{definition}
   [pseudo-dimension \citep{pollard1990empirical}]\label{Pse}
		Let $\fF$ be a class of functions from $\fX$ to $\sR$. The pseudo-dimension of $\fF$, denoted by $\text{Pdim}(\fF)$, is the largest integer $m$ for which there exists $(x_1,x_2,\ldots,x_m,y_1,y_2,\ldots,y_m)\in\fX^m\times \sR^m$ such that for any $(b_1,\ldots,b_m)\in\{0,1\}^m$ there is $f\in\fF$ such that $\forall i: f\left(x_i\right)>y_i \Longleftrightarrow b_i=1.$    
   \end{definition}

\begin{lemma}[\cite{anthony1999neural}, Theorem 12.2]\label{cover dim}
				Let $\fF$ be a class of functions from $\fX$ to $[-B,B]$. For any $\varepsilon>0$, we have \[\fN(\varepsilon,\fF,n)\le \left(\frac{2\mathrm{e}nB}{\varepsilon\text{Pdim}(\fF)}\right)^{\text{Pdim}(\fF)}\] for $n\ge \text{Pdim}(\fF)$.
			\end{lemma}

We estimate $\operatorname{Pdim}\!\bigl(\partial_k\mathcal{F}_{m,L}\bigr)$ by adapting the argument of \cite[Theorem 2]{yang2023nearly,yang2024deeper}. The proof is slightly adapted for readability is provided in  \ref{app:pdim} for completeness.

\begin{proposition}[Pseudo-dimension of derivative network classes]
\label{prop:pdim}
Fix \(m,N,L\in\mathbb N_+\). Let \(\Phi_2\) be the class of all functions that
can be written as a finite sum of \(m\) squared ReLU subnetworks, each with
depth at most \(L\) and width at most \(N\). Similarly, let \(\Phi_1\) be the
class of all functions that can be written as a finite sum of \(m\) ReLU
subnetworks, each with depth at most \(L\) and width at most \(N\). For
\(\kappa=1,2\) and \(k=1,\dots,d\), define
\[
\partial_k\Phi_\kappa
:=
\{\partial_k f:\ f\in\Phi_\kappa\}.
\]
Then there exists an absolute constant \(C>0\), independent of \(m,N,L\), and
\(d\), such that
\[
\operatorname{Pdim}(\partial_k\Phi_\kappa)
\le
C\,m\,L^{1+\kappa}N^2
\log_2(2N)\log_2(2L)\log_2(2m).
\]
\end{proposition}

\begin{proof}[Proof of Theorem~\ref{general thm}]
By Proposition~\ref{prop:pdim}, for \(\kappa=1,2\) and each \(k=1,\dots,d\),
\[
\operatorname{Pdim}(\partial_k\mathcal F^\kappa_{m,L,C_*})
\le
C\,\operatorname{poly}(d)\,m(\log m)^3,
\]
up to logarithmic factors of order \(\log\log m\). In particular, the same bound
applies to the derivative class used in the following estimate, after enlarging
the constant if necessary.

By Lemma~\ref{cover dim} and Proposition~\ref{prop:connect}, we obtain
\begin{align}
&\mathbb E
\left\|
f_{\mathcal S,\fF_{m,L,C_*}^\kappa}-f_\rho
\right\|_E^2
\notag\\
&\le
\sum_{k=1}^d
\frac{5136L^4}{M}
\left[
\log\left(
14\mathcal N\left(
\frac{1}{80LM},
\partial_k\fF_{m,L,C_*}^\kappa,
M
\right)
\right)+1
\right]
+
2\inf_{f\in\fF_{m,L,C_*}^\kappa}
\|f-f_\rho\|_E^2
\notag\\
&\le
\sum_{k=1}^d
\frac{5136L^4}{M}
\left[
\log\left(
28\mathcal N\left(
\frac{1}{80LM},
\partial_k\mathcal F^\kappa_{m,L,C_*},
M
\right)
\right)+1
\right]
+
C|f_\rho|_{2,2}^2m^{-2}.
\label{eq:generalization-step-1}
\end{align}
Here we used the approximation estimate from
Corollary~\ref{cor:gen-error-form} for the approximation error term.

Applying the covering-number bound in terms of pseudo-dimension gives
\begin{align}
&\mathbb E
\left\|
f_{\mathcal S,\fF_{m,L,C_*}^\kappa}-f_\rho
\right\|_E^2
\notag\\
&\le
\sum_{k=1}^d
\frac{5136L^4}{M}
\left[
\operatorname{Pdim}(\partial_k\fF_{m,L,C_*}^\kappa)
\log\left(
\frac{160eM^2L^2}
{\operatorname{Pdim}(\partial_k\fF_{m,L,C_*}^\kappa)}
\right)
+4
\right]
+
C|f_\rho|_{2,2}^2m^{-2}
\notag\\
&\le
\frac{C\,\operatorname{poly}(d,L)}{M}
\left[
m(\log m)^3\log M+1
\right]
+
C|f_\rho|_{2,2}^2m^{-2}.
\end{align}
Therefore,
\begin{equation}
\mathbb E
\left\|
f_{\mathcal S,\fF_{m,L,C_*}^\kappa}-f_\rho
\right\|_E^2
\le
C\,\operatorname{poly}(d,L,|f_\rho|_{2,2})
\left(
\frac{m(\log m)^3\log M}{M}
+
m^{-2}
\right).
\label{eq:generalization-balance-before}
\end{equation}

We now choose \(M\) by balancing the estimation error and approximation error.
More precisely, we take
\[
M=\lceil m^3(\log m)^4\rceil .
\]
Then
\[
\frac{m(\log m)^4}{M}
\leq
\frac{m(\log m)^4}{m^3(\log m)^4}
=
m^{-2}.
\]
Therefore, \eqref{eq:generalization-balance-before} gives
\[
\mathbb E
\left\|
f_{\mathcal S,\fF_{m,L,C_*}^\kappa}-f_\rho
\right\|_E^2
\le
C\,\operatorname{poly}(d,L,|f_\rho|_{2,2})\,m^{-2}.
\]
Since \(M=\lceil m^3(\log m)^4\rceil\), we have
\[
m^{-2}
\leq
\left(\frac{(\log m)^4}{M}\right)^{2/3}.
\]
Moreover, \(\log m\le C\log M\). Hence
\[
m^{-2}
\le
C\left(\frac{(\log M)^4}{M}\right)^{2/3}.
\]
Consequently,
\[
\mathbb E
\left\|
f_{\mathcal S,\fF_{m,L,C_*}^\kappa}-f_\rho
\right\|_E^2
\le
C\,\operatorname{poly}(d,L,|f_\rho|_{2,2})
\left(\frac{(\log M)^4}{M}\right)^{2/3}.
\]
This proves the desired generalization error bound.
\end{proof}

\section*{Acknowledgment} YL thanks the support from the NSF CAREER Award DMS-2442463. TM and JX thank the support from KAUST Baseline Research Fund.


\bibliographystyle{elsarticle-num-names}
\bibliography{references}

\appendix
\section{Proof of Classical Multiplication by ReLU Neural Networks}
\begin{lemma}[Product approximation by ReLU neural networks]
\label{lem:product-relu}
Let \(\sigma(x):=\max\{x,0\}\). For any \(M,K,d\in\mathbb N_+\) with \(d\ge2\),
there exists a ReLU neural network \(\Phi_{M,K,d}:[-1,1]^d\to\mathbb R\) with
width
\[
N\le 9(M+1)+d-1
\]
and depth
\[
L\le 14d(d-1)K
\]
such that
\begin{equation}
\left\|
\Phi_{M,K,d}(\vx)-x_1x_2\cdots x_d
\right\|_{W^{1,\infty}([-1,1]^d)}
\le
10(d-1)(M+1)^{-7dK}.
\label{eq:prod-relu-error}
\end{equation}
Moreover, for every \(i=1,\dots,d\), if \(x_i=0\), then
\begin{equation}
\Phi_{M,K,d}(x_1,\dots,x_{i-1},0,x_{i+1},\dots,x_d)=0.
\label{eq:prod-relu-zero}
\end{equation}
In addition, all weights and biases of \(\Phi_{M,K,d}\) can be chosen to have
magnitude bounded by an absolute constant independent of \(M\), \(K\), and
\(d\).
\end{lemma}

\begin{proof}
The approximation estimate \eqref{eq:prod-relu-error} follows from the
recursive product construction in \cite[Proposition 5]{yang2023nearly}. Since
the parameter bound is not stated explicitly there, we briefly explain why the
parameters can be chosen uniformly bounded.

We first consider the two-variable case. Define the tent map \(T_1\) on
\([-1,1]\) by
\[
T_1(\widetilde x)
=
\begin{cases}
2|\widetilde x|, & |\widetilde x|\le \frac12,\\
2(1-|\widetilde x|), & |\widetilde x|>\frac12,
\end{cases}
\]
and recursively define \(T_i=T_{i-1}\circ T_1\) for \(i=2,3,\dots\). For
\(M,K\in\mathbb N_+\), choose \(\widetilde L=kK\), where \(k\) is determined by
\((k-1)2^{k-1}+1\le M\le k2^k\). Define
\[
\widetilde\psi(\widetilde x)
=
\widetilde x-\sum_{i=1}^{\widetilde L}\frac{T_i(\widetilde x)}{2^{2i}}.
\]
By \cite[Lemma 3.2]{hon2022simultaneous} and
\cite[Lemma 5.1]{lu2021deep}, \(\widetilde\psi\) can be realized by a ReLU
network with width \(3M\) and depth \(2K\), and satisfies
\[
\|\widetilde\psi(\widetilde x)-\widetilde x^2\|_{W^{1,\infty}([-1,1])}
\le M^{-K},
\qquad
\widetilde\psi(0)=0.
\]
Moreover, the construction uses only the tent map \(T_1\), its compositions,
and coefficients \(2^{-2i}\), all of which can be realized using uniformly
bounded weights and biases. Hence all parameters are bounded by an absolute
constant independent of \(M\) and \(K\).

Now define
\begin{equation}
\Phi_{2,M,K}(x,y)
=
2\left[
\widetilde\psi\!\left(\frac{|x+y|}{2}\right)
-
\widetilde\psi\!\left(\frac{|x|}{2}\right)
-
\widetilde\psi\!\left(\frac{|y|}{2}\right)
\right].
\label{eq:phi2-def}
\end{equation}
Then \(\Phi_{2,M,K}\) is a ReLU network with width \(9M\) and depth \(2K\), and
\[
\|\Phi_{2,M,K}(x,y)-xy\|_{W^{1,\infty}([-1,1]^2)}
\le 6M^{-K}.
\]
Since \(\widetilde\psi(0)=0\), the network also satisfies
\(\Phi_{2,M,K}(0,y)=\Phi_{2,M,K}(x,0)=0\). Furthermore, all weights and biases
remain uniformly bounded, since \eqref{eq:phi2-def} only uses addition,
absolute value, scaling by \(1/2\), and the network \(\widetilde\psi\).

For \(d\ge2\), the \(d\)-fold product network is constructed recursively as in
\cite[Proposition 5]{yang2023nearly}, by repeatedly applying the two-variable
product approximation. This recursive construction gives a ReLU network with
width \(9(M+1)+d-1\), depth \(14d(d-1)K\), and error bounded by
\(10(d-1)(M+1)^{-7dK}\) in \(W^{1,\infty}([-1,1]^d)\). The vanishing property
follows from the zero-preserving property of the two-variable block. Since the
recursive composition uses only subnetworks whose parameters are uniformly
bounded, all weights and biases of the final network are bounded by an absolute
constant independent of \(M\), \(K\), and \(d\). This completes the proof.
\end{proof}
\section{Proof of Proposition \ref{prop:pdim}}\label{app:pdim}
Before we estimate pseudo-dimension of $\partial_k\Phi$ for any $k=1,2,\ldots,d$, we first introduce Vapnik--Chervonenkis dimension (VC-dimension) \begin{definition}[VC-dimension \cite{abu1989vapnik}]
		Let $H$ denote a class of functions from $\fX$ to $\{0,1\}$. For any non-negative integer $m$, define the growth function of $H$ as \[\Pi_H(m):=\max_{x_1,x_2,\ldots,x_m\in \fX}\left|\{\left(h(x_1),h(x_2),\ldots,h(x_m)\right): h\in H \}\right|.\] The Vapnik--Chervonenkis dimension (VC-dimension) of $H$, denoted by $\text{VCdim}(H)$, is the largest $m$ such that $\Pi_H(m)=2^m$. For a class $\fG$ of real-valued functions, define $\text{VCdim}(\fG):=\text{VCdim}(\operatorname{sgn}(\fG))$, where $\operatorname{sgn}(\fG):=\{\operatorname{sgn}(f):f\in\fG\}$ and $\operatorname{sgn}(x)=1[x>0]$.\end{definition}

    In the proof of Proposition \ref{prop:pdim}, we use the following lemmas:
		\begin{lemma}[{{\cite{bartlett2019nearly,anthony1999neural}}}]\label{bounded}
			Suppose $W\le M$ and let $P_1,\ldots,P_M$ be polynomials of degree at most $D$ in $W$ variables. Define \[K:=\left|\{\left(\operatorname{sgn}(P_1(a)),\ldots,\operatorname{sgn}(P_M(a))\right):a\in\sR^W\}\right|,\] then we have $K\le 2(2eMD/W)^W$.
		\end{lemma}
		
		\begin{lemma}[{\cite{bartlett2019nearly}}]\label{inequality}
			Suppose that $2^m\le 2^t(mr/w)^w$ for some $r\ge 16$ and $m\ge w\ge t\ge0$. Then, $m\le t+w\log_2(2r\log_2r)$.
		\end{lemma}

  \begin{proof}[Proof of Proposition~\ref{prop:pdim}]
We first prove the estimate for the squared ReLU class \(\Phi_2\). The proof
for the ReLU class \(\Phi_1\) is analogous and slightly simpler, since the
derivatives of ReLU subnetworks have lower algebraic complexity than those of
squared ReLU subnetworks. We explain the resulting modification at the end of
the proof.

For any $k=1,2,\ldots,d$, denote \[\Phi_{\fN}:=\{\eta(\vx,y):\eta(\vx,y)=\psi(\vx)-y,\psi\in \partial_k\Phi, (\vx,y)\in\sR^{d+1}\}.\]
			Based on the definition of VC-dimension and pseudo-dimension, we have that\begin{equation}
				\text{Pdim}(\partial_k\Phi)\le \text{VCdim}(\Phi_{\fN}).
			\end{equation}
			For the $\text{VCdim}(\Phi_{\fN})$, it can be bounded by following way. The proof is similar to that in \cite{yang2023nearly}. Let $\vz=(\vx,y)\in\sR^{d+1}$ be an input and $\vtheta\in\sR^W$ be a parameter vector in $\eta:=\psi^2-y$. We denote the output of $\psi$ with input $\vz$ and parameter vector $\vtheta$ as $f(\vz,\vtheta)$. For fixed $\vz_1=(\vx_1,y_1),\vz_2=(\vx_2,y_2),\ldots,\vz_P=(\vx_P,y_P)$ in $\sR^d$, we aim to bound\begin{align}
				K:=\left|\{\left(\operatorname{sgn}(f(\vz_1,\vtheta)),\ldots,\operatorname{sgn}(f(\vz_P,\vtheta))\right):\vtheta\in\sR^W\}\right|.
			\end{align}
			
			For any partition $\fS=\{P_1,P_2,\ldots,P_T\}$ of the parameter domain $\sR^W$, we have \[K\le \sum_{i=1}^T\left|\{\left(\operatorname{sgn}(f(\vz_1,\vtheta)),\ldots,\operatorname{sgn}(f(\vz_P,\vtheta))\right):\vtheta\in P_i\}\right|.\] We choose the partition such that within each region $P_i$, the functions $f(\vz_j,\cdot)$ are all fixed polynomials of bounded degree. This allows us to bound each term in the sum using Lemma \ref{bounded}.
   
			First, we consider $m=1$, for a squared ReLU neural networks $\phi\in\Phi$ with $m=1$, it can be represented as \[\phi=\vW_{L+1}\sigma(\vW_{L}\sigma(\ldots\sigma(\vW_{1}\vx+\vb_{1})\ldots)+\vb_{L})+b_{L+1}.\] Therefore, set $\psi(\vx):=\partial_k \phi(\vx)$ \begin{align}
				\psi=&\vW_{L+1}\sigma'(\vW_{L}\sigma(\ldots\sigma(\vW_{1}\vx+\vb_{1})\ldots)+\vb_{L})\notag\\&\cdot \vW_{L}\sigma'(\ldots\sigma(\vW_{1}\vx+\vb_{1})\ldots)\ldots\vW_{1}\sigma'(\vW_{1}\vx+\vb_{1})(\vW_{1})_k,
			\end{align}where $\vW_n\in\sR^{N_n\times N_{n-1}}$ ($(\vW)_i$ is $i$-th column of $\vW$) and $\vb_n\in\sR^{N_n}$ are the weight matrix and the bias vector in the $n$-th linear transform in $\phi$, and $~\sigma'(\vx)=\operatorname{diag}(\sigma'(x_i))$. Denote $W_n$ as the number of parameters in $\vW_n,\vb_n$, i.e., $W_n=N_nN_{n-1}+N_{n}\le 2N^2$.

			We define a sequence of sets of functions $\{\sF_j\}_{j=0}^{L}$ with respect to parameters $\vtheta \in \mathbb{R}^W$:\begin{align}
				\overline\sF_{0}&:=\{(\vW_{1})_k\}\cup\{\vW_{1}\vx+\vb_{1}\}\notag\\
				\overline\sF_{1}&:=\{(\vW_{1})_k\}\cup\{\vW_{2}\sigma'(\vW_{1}\vx+\vb_{1}),\vW_{2}\sigma(\vW_{1}\vx+\vb_{1})+\vb_{2}\} \\\notag\\&\vdots\notag\\\overline\sF_{L}&:=\{(\vW_{1})_k\}\cup\{\vW_{2}\sigma'(\vW_{1}\vx+\vb_{1}),\ldots,\vW_{L+1}\sigma'(\vW_{L}\sigma(\ldots\sigma(\vW_{1}\vx+\vb_{1})\ldots)+\vb_{L})\}\notag\\\sF_n&:=\bigcup_{p=1}^P\overline\sF_n|_{\vx=\vx_p}.\notag
			\end{align}
			
			The partition of $\sR^W$ is constructed layer by layer through successive refinements denoted by $\fS_0,\fS_1,\ldots,\fS_L$. These refinements possess the following properties:

            \textbf{1}. We have $|\fS_0|=1$, and for each $n=1,\ldots,L$, we have $\frac{|\fS_n|}{|\fS_{n-1}|}\le 2\left(\frac{2eP(1+(n-1)2^{n-1})N_{n}}{\sum_{i=1}^nW_{i}}\right)^{\sum_{i=1}^nW_{i}}$.
			
			\textbf{2}. For each $n=0,\ldots,L-1$, each element $S$ of $\fS_{n}$, when $\vtheta$ varies in $S$, the output of each term in $\sF_n$ is a fixed polynomial function in $\sum_{i=1}^nW_i$ variables of $\vtheta$, with a total degree no more than $1+n2^{n}$.

            \textbf{3}. For each element $S$ of $\fS_{L}$, when $\vtheta$ varies in $S$, every term in $\sF_L$ can be expressed as a polynomial function of at most degree $1+L2^{L}$, which depends on at most $\sum_{i=1}^{L+1}W_i$ variables of $\vtheta$.
			
			We define $\fS_0=\{\sR^W\}$, which satisfies properties 1,2 above, since $\vW_{1}\vx_p+\vb_{1}$ and $(\vW_{1})_k$ are affine functions of $\vW_{1},\vb_{1}$ for any $p=1,\ldots,P$ and $i=1,\ldots,d$.
			
			To define $\fS_n$, for the fixed $h$, we use the last term of $\overline\sF_{n-1}$ as inputs for the last two terms in $\overline\sF_{n}$. Assuming that $\fS_0,\fS_1,\ldots,\fS_{n-1}$ have already been defined, for each $h\in [Z]$, we observe that the last two terms are new additions to $\overline\sF_{n}$ when comparing it to $\overline\sF_{n-1}$. Therefore, all elements in $\sF_{n}$ except the last two of \(\overline\sF_n|_{\vx=\vx_p}\) for each $p=1,\ldots,P$ are fixed polynomial functions in $\sum_{i=1}^nW_{i}$ variables of $\vtheta$, with a total degree no greater than $1+(n-1)2^{n-1}$ when $\vtheta$ varies in $S\in\fS_n$. This is because $\fS_n$ is a finer partition than $\fS_{n-1}$.
			
			We denote $p_{\vx_p,n-1,S,k}(\vtheta)$ as the output of the $k$-th node in the last term of $\overline\sF_{n-1}$ in response to $\vx_p$ when $\vtheta\in S$. The collection of polynomials \[\sH_{n-1}:=\{p_{\vx_p,n-1,S,k}(\vtheta): p=1,\ldots,P,~k=1,\ldots,N_n\}\]can attain at most $2\left(\frac{2eP(1+(n-1)2^{n-1})N_{n}}{\sum_{i=1}^nW_{i}}\right)^{\sum_{i=1}^nW_{i}}$ distinct sign patterns when $\vtheta\in S$ due to Lemma \ref{bounded} for sufficiently large $m$. Therefore, we can divide $S$ into $2\left(\frac{2eP(1+(n-1)2^{n-1})N_{n}}{\sum_{i=1}^nW_{i}}\right)^{\sum_{i=1}^nW_{i}}$ parts, each does not change sign within the subregion. By performing this for all $S\in\fS_{n-1}$, we obtain the desired partition $\fS_n$. This division ensures that the required property 1 is satisfied.
			
			Additionally, since the input to the last two terms in $\overline\sF_{n}|_{\vx=\vx_p}$ for each $p$ is in $\sH_{n-1}$, and we have shown that the sign of this input will not change in each region of $\fS_n$, it follows that the output of the last two terms in $\overline\sF_{n}|_{\vx=\vx_p}$ for each $p$ is also a polynomial without breakpoints in each element of $\fS_n$. Therefore, the required property 2 is satisfied by check the degree of the polynomial directly.

            For any element $S$ of the partition $\fS_{L}$, when the vector of parameters $\vtheta$ varies within $S$, every term in $\sF_L$ can be expressed as a polynomial function of at most degree $1+L2^{L}$, which depends on at most $\sum_{i=1}^{L+1}W_i$ variables of $\vtheta$. Hence, for any $\vz_p$, $p=1,\ldots.P$, $f(\vz_p;\vtheta)$ can be expressed as a polynomial function of at most degree $L+L^22^{L}$, which depends on at most $\sum_{i=1}^{L+1}W_i$ variables of $\vtheta$. Hence, the required property 3 is satisfied. Furthermore, the number of the subregion is \[\prod_{n=1}^L2\left(\frac{2eP(1+(n-1)2^{n-1})N_{n}}{\sum_{i=1}^nW_{i}}\right)^{\sum_{i=1}^nW_{i}}.\]

            For $m>1$ we repeat the above construction for each of the $m$ subnetworks and take the common refinement, obtaining a partition $\mathcal S_{*}$ of the joint parameter space.  
For every cell $S\in\mathcal S_{*}$ and every input $\boldsymbol z_p$, the realization
\(
  f(\boldsymbol z_p;\boldsymbol\theta)
\)
is a polynomial in the parameters $\boldsymbol\theta$ of total degree at most
\(
  L+L^{2}2^{L},
\)
depending on no more than
\(
  m\sum_{i=1}^{L+1}W_{i}
\) variables of $\vtheta$. Furthermore, the number of the subregion is \[\prod_{n=1}^L2\left(\frac{2eP(1+(n-1)2^{n-1})N_{n}}{\sum_{i=1}^nW_{i}}\right)^{m\sum_{i=1}^nW_{i}}.\] 
			
            
            Therefore, for each $S\in\fS_*$ we have \begin{align*}&\left|\{\left(\operatorname{sgn}(f(\vx_1,\vtheta)),\ldots,\operatorname{sgn}(f(\vx_P,\vtheta))\right):\vtheta\in S\}\right|\le 2\left(2eP(L+L^22^{L})/m\sum_{i=1}^{L+1}W_i\right)^{m\sum_{i=1}^{L+1}W_i}.\end{align*} Then \begin{align}
				K\le& 2\left(2eP(L+L^22^{L})/m\sum_{i=1}^{L+1}W_i\right)^{m\sum_{i=1}^{L+1}W_i}\cdot\prod_{n=1}^L2\left(\frac{2eP(1+(n-1)2^{n-1})N_{n}}{\sum_{i=1}^nW_{i}}\right)^{m\sum_{i=1}^nW_{i}}\notag\\\le &2^{L+1}\left(2emP(2L+2L^22^{L})/m\sum_{n=1}^{L+1}\sum_{i=1}^{n}W_i\right)^{m\sum_{n=1}^{L+1}\sum_{i=1}^{n}W_i}
			\end{align}where the inequality is due to weighted AM-GM. For the definition of the VC-dimension, we have \begin{equation}
				2^{\text{VCdim}(\Phi_{\fN})}\le 2^{ L+1}\left(2e\text{VCdim}(\Phi_{\fN})m(2L+2L^22^{L})/U\right)^{U}
			\end{equation}where $U=m\sum_{n=1}^{L+1}\sum_{i=1}^{n}W_i\le C_1mN^2L^2$ where $C_1$ is independent of $N,L,d,m$. Due to Lemma \ref{inequality}, we obtain that\begin{align}
				\text{VCdim}(\Phi_{\fN})\le CmL^3 N^2\log_2N\log_2 L\log_2m,\notag
			\end{align} where $C$ is independent of $N,L,d,m$.

            The argument above proves the estimate for the squared ReLU class \(\Phi_2\).
For the ReLU class \(\Phi_1\), the proof is analogous, but slightly simpler.
The only difference appears in the iterative construction used in Properties
1--3. In the squared ReLU case, the number of polynomial pieces may double at
each layer, leading to the additional factor \(2^n\) in the corresponding
piece-count estimate. For ReLU networks, however, the relevant functions are
piecewise linear, and the complexity of each piece remains linear. Therefore,
this additional factor \(2^n\) is replaced by \(1\) in the ReLU case. As a
result, the pseudo-dimension bound for \(\Phi_1\) contains \(L^2\) instead of
\(L^3\).
		\end{proof}
\end{document}